\documentclass{article}

\usepackage[preprint]{neurips_2026}

\usepackage{nicefrac}       
\usepackage{xcolor}         
\usepackage{amsmath,amssymb,amsthm,graphicx,enumitem,hyperref}
\usepackage{mathtools}
\usepackage{tikz}
\usepackage{float}
\usepackage{booktabs}
\usepackage{bm}
\usepackage{algorithm}
\usepackage{algorithmic}
\usepackage{multirow} 
\usepackage{graphbox}

\usepackage{listings}

\definecolor{codegreen}{rgb}{0,0.6,0}
\definecolor{codegray}{rgb}{0.5,0.5,0.5}
\definecolor{codepurple}{rgb}{0.58,0,0.82}
\definecolor{backcolour}{rgb}{0.95,0.95,0.92}

\lstdefinestyle{mystyle}{
    backgroundcolor=\color{backcolour},   
    commentstyle=\color{codegreen},
    keywordstyle=\color{magenta},
    numberstyle=\tiny\color{codegray},
    stringstyle=\color{codepurple},
    basicstyle=\ttfamily\footnotesize,
    breakatwhitespace=false,         
    breaklines=true,                 
    captionpos=b,                    
    keepspaces=true,                 
    numbers=left,                    
    numbersep=5pt,                  
    showspaces=false,                
    showstringspaces=false,
    showtabs=false,                  
    tabsize=2,
    language=Python
}

\newtheorem{theorem}{Theorem}
\newtheorem{lemma}{Lemma}
\newtheorem{proposition}{Proposition}
\newtheorem{corollary}{Corollary}
\newtheorem{definition}{Definition}

\newcommand{\Rpn}{\mathbb{R}_+^n}
\newcommand{\Rppn}{\mathbb{R}_{++}^n}
\newcommand{\1}{\mathbf{1}}

\newcommand{\IS}{\mathsf{IS}}
\newcommand{\MIS}{\mathsf{MIS}}
\newcommand{\supp}{\operatorname{supp}}
\newcommand{\diag}{\operatorname{diag}}
\newcommand{\adj}{\operatorname{adj}}
\newcommand{\spec}{\operatorname{spec}}

\newcommand{\GN}[1]{\mathcal{N}_{#1}}
\newcommand{\NA}{\mathsf{N}_A}
\newcommand{\CLO}[1]{{#1}_{\scriptscriptstyle I}}

\newcommand{\WRAM}{A^{\gamma, v}}

\newcommand{\ENv}{\mathcal{E}_{\gamma, v}}
\newcommand{\TENv}{\tilde{\mathcal{E}}_{\gamma, v}}
\newcommand{\stab}{\operatorname{stab}}

\title{Graph Normalization: Fast Binarizing Dynamics for Differentiable MWIS}
\author{Laurent Guigues \\ \texttt{laurent.guigues@gmail.com}}
\date{\today}

\begin{document}
\maketitle

\begin{abstract}
We introduce Graph Normalization (GN), 
a principled dynamical system on graphs
that serves as a differentiable approximation engine for the NP-hard Maximum Weight Independent Set (MWIS) problem. 
MWIS encompasses a vast class of combinatorial challenges, 
including optimal assignment, scheduling, set packing, and MAP inference in discrete Markov Random Fields. 
Unlike Belief Propagation, we prove that GN always converges to a binary indicator of a Maximum Independent Set (MIS). 
GN realizes a fast quasi-Newton descent through an exact Majorization-Minimization step, 
which systematically improves the MWIS relaxed primal objective. 
We establish an equivalence between GN and the Replicator Dynamics 
of a nonlinear evolutionary game, where vertices compete for inclusion in an independent set. 
While being a non-potential game, the GN game follows Fisher’s Fundamental Theorem of Natural Selection, 
where the average fitness is equal to the MWIS primal objective and strictly increases over time. 
This connection leads to a weighted extension of the Motzkin-Straus theorem, 
showing that MISes are in bijection with the local minima of a quadratic form over a tilted simplex. 
For the Assignment Problem, GN acts as a variant of the Sinkhorn algorithm 
that naturally converges to a hard assignment 
while generalizing to arbitrary constraint graphs beyond bipartite matching. 
We demonstrate GN's performance as a fast binarization engine 
for the state-of-the-art Bregman-Sinkhorn relaxed MWIS solver. 
On real-world benchmarks with up to 1M edges, GN identifies solutions within a $1$\% gap of the best known results 
in seconds on a laptop CPU. 
GN opens new avenues for deep learning architectures requiring differentiable, 
"hard" decisions under constraints, with potential applications in structured sparse attention, 
dynamic network pruning, and Mixture-of-Experts. 
Beyond core AI, the GN framework provides a foundation for end-to-end learning 
of constrained optimization tasks in fields such as computer vision, computational biology, and resource allocation.
\end{abstract}

\section{Introduction}
The Maximum Weight Independent Set (MWIS) problem 
-- selecting a set of non-adjacent nodes in a graph of maximum total weight -- 
is a cornerstone of combinatorial optimization.
It is equivalent to the Maximum Weight Clique (MWC) problem in the complementary graph 
and it directly subsumes classic problems such as Bipartite Graph Matching and its $n$-dimensional extensions, 
$k$-Coloring, Maximum Weight Set Packing, the multi-dimensional Knapsack problem, and Inexact Graph Matching \cite{bunke1997relation}. 
Furthermore, it serves as a pivotal framework for MAP inference in discrete Markov Random Fields \cite{sanghavi2009message}, 
Max-SAT reductions \cite{karp1972reducibility}, and Distributed Constraint Optimization (DCOP) \cite{fioretto2018constraint}. 
Many real-world problems can be formulated as MWIS instances, 
spanning from logistics and vehicle routing \cite{dumas1991pickup},
wireless link scheduling in telecommunications \cite{tassiulas1990stability},
to map label placement \cite{verweij1999optimisation}. 
In the domain of computer vision, this includes image segmentation \cite{brendel2010segmentation}, 
3D reconstruction \cite{zhuo2016independent}, multi-object tracking \cite{papageorgiou2009maximum}, 
and action recognition \cite{yu2015fast}. 
In biological sciences, MWIS addresses molecular sequence alignment \cite{kececioglu1993maximum}, 
drug discovery \cite{banchi2020molecular}
and genome sequence assembly \cite{laso2020metavar}, 
while in economics, it provides the foundation for solving combinatorial auctions \cite{de2003combinatorial} 
and is used to model portfolio management \cite{bettinelli2017branch}.

Despite its fundamental importance, MWIS is NP-hard and hard to approximate \cite{garey1979computers}, 
making exact solutions intractable for the massive real-world graphs encountered in modern applications. 
Traditional approaches often rely on LP relaxations of the associated Binary Linear Program (BLP) \cite{haller2024}; 
however, bridging the gap between fractional solutions and feasible integer assignments,
i.e., actual solutions to the MWIS problem, remains a significant challenge and often a computational bottleneck.
Even modern learning-based approaches using Graph Neural Networks (GNNs), such as \cite{karalias2020erdos}, 
typically solve a continuous relaxation of the problem and suffer from the same binarization deficit: 
they do not guarantee a valid binary output, necessitating post-hoc greedy heuristics. 
They also fail to provide the rigorous optimality gap guarantees found in classical optimization frameworks. 

Another fundamental challenge is the integration of MWIS solvers into deep learning pipelines 
to enable end-to-end learning of constrained tasks. 
Classical combinatorial algorithms are procedural and non-differentiable, 
precluding their direct use in gradient-based learning. 
While black-box differentiation \cite{poganvcic2019differentiation} circumvents this by perturbing the input, 
it incurs high computational overhead and treats the solver as an external oracle rather than a native dynamical layer. 
Alternatively, Reinforcement Learning (RL) has been employed 
to treat optimization as a sequential decision task \cite{bello2016neural, khalil2017learning}; 
however, RL suffers from significant sample inefficiency and high variance, 
making it computationally prohibitive for the production-scale graphs addressed in this work. 
While Belief Propagation (BP) is differentiable, it is notoriously prone to non-convergence on graphs with loops  \cite{sanghavi2009message}. 
To our knowledge, the only native differentiable MWIS framework providing guaranteed convergence to a valid binary output 
without external annealing schedules is the Replicator Dynamics approach of Bomze and co-workers \cite{bomze1997}. 
We detail the relationship between our approach and Bomze's 
-- and why GN provides superior scalability -- in Section \ref{sec:game}.

Crucially, the most prominent success story of differentiable combinatorial optimization
-- the Sinkhorn-Knopp (SK) algorithm -- is fundamentally an MWIS solver restricted to bipartite graphs.
Indeed, Bipartite Graph Matching (BGM), 
also known as the Assignment Problem 
-- finding a 1:1 matching between two sets which maximizes the total weight of the matches, 
is equivalent to MWIS in the line graph $L(K_{n,n})$, 
in which the role of nodes and edges is exchanged compared to the complete bipartite graph $K_{n,n}$.
BGM is an easy sub-problem of MWIS: it can be solved in polynomial time, 
e.g., by Kuhn's Hungarian algorithm \cite{kuhn1955hungarian}.
In their seminal work, Sinkhorn and Knopp have proven that alternating row and column normalization 
of a non-negative matrix with total support converges to a bi-stochastic matrix \cite{sinkhorn1967concerning}, 
i.e., to a "soft" / relaxed solution of the Assignment Problem.
The relationship between Sinkhorn-Knopp and the "hard" assignment problem 
was historically navigated through the lens of statistical physics, 
most notably in the SoftAssign \cite{gold1996softmax} 
and Graduated Assignment \cite{gold1996graduated} frameworks. 
These methods utilize deterministic annealing to solve the assignment problem 
by minimizing a free energy function \cite{yuille1990stereo, kosowsky1994invisible}; 
as the ``temperature'' parameter is decreased toward the $0$-temperature limit, 
the soft solution is forced to converge toward a ``crisp'' permutation matrix \cite{rangarajan1997convergence}. 
This lineage provided the foundational theory for 
the modern breakthrough of Optimal Transport (OT), 
where the Sinkhorn algorithm provides an efficient solution 
to a strictly convex relaxation of the Kantorovich problem \cite{cuturi2013sinkhorn}. 
Mathematically, Sinkhorn iterations act as Bregman projections 
-- minimizing the Kullback-Leibler divergence toward the constraint set -- 
where an entropic penalty effectively acts 
as the inverse of the annealing temperature \cite{cuturi2013sinkhorn, genevay2019entropy}.
The popularization of OT
within the machine learning community has, in turn, 
established ``Sinkhorn Layers'' as foundational plug-and-play modules 
across a diverse array of neural architectures 
-- ranging from differentiable ranking \cite{adams2011rankingPreprint, mena2018learning}, graph alignment \cite{zanfir2018deep}, 
image matching \cite{sarlin2020superglue} to multi-person 3D pose tracking \cite{reddy2021tessetrack}.
The utility of Sinkhorn dynamics is evident in modern Large Language Models (LLMs);
 Mixture-of-Experts (MoE) and Manifold Constrained Hyper-Connections \cite{xie2025mhc} 
 employ SK iterations to stabilize training and enforce balanced token routing \cite{clark2022unified, deepseek2024v3}.

Despite its success, SK is fundamentally constrained by its underlying bipartite topology, 
addressing primarily the assignment problem and its 
multi-dimensional generalization \cite{altschuler2023polynomial}.
While effective for 1:1 matchings, SK cannot natively resolve complex multi-way conflicts found in general relational data 
without artificial bipartitization such as in the SoftAssign algorithm for graph matching \cite{gold1996graduated}. 
Furthermore, SK typically produces soft assignments within the Birkhoff polytope, 
requiring vanishing entropy regularization via sensitive annealing schedules to recover discrete solutions \cite{rangarajan1999convergence}. 

In this paper, we introduce Graph Normalization (GN), 
a principled dynamical system that overcomes these topological constraints by generalizing Sinkhorn-like 
normalization dynamics to arbitrary constraint graphs. 
We show that GN realizes a fast quasi-Newton descent through an exact Majorization-Minimization step 
that is mathematically proven to always converge to a binary indicator of a Maximum Independent Set, 
eliminating the need for post-hoc heuristics or sensitive annealing schedules.
We establish a formal equivalence between GN and the Replicator Dynamics of a nonlinear evolutionary game, 
where the average population fitness identifies with the MWIS primal objective and is proven to strictly increase over time. 
Empirically, we demonstrate the efficiency of GN as as a high-speed binarization engine 
for the relaxed solutions provided by the state-of-the-art Bregman-Sinkhorn fractional MWIS algorithm \cite{haller2024}, 
producing high quality solutions on production-scale graphs with millions of edges. 
As a \emph{differentiable} MWIS engine, 
GN unlocks end-to-end learning of complex constrained tasks, 
bridging the gap between combinatorial optimization, 
theoretical game theory, and the structural demands of modern neural architectures.

\section{Graph Normalization: Definition and Basic Properties}

We denote $\Rpn$ and $\Rppn$ as the non-negative and positive orthants, respectively. 
We write $x\ge 0$ if $x\in\Rpn$ and $x>0$ if $x\in\Rppn$.
For $x,y \in \Rppn$, $x \odot y$ and $x \oslash y$ denote resp. componentwise 
multiplication and division.
Let $G=(V,E)$ be a simple undirected graph with $|V|=n$ 
and adjacency matrix $A \in \{0,1\}^{n \times n}$, verifying $A_{ii}=0$ as $G$ is simple.
We identify $V$ with $\{1\dots n\}$ and $G$ with its adjacency matrix.
For a vector $x\in\Rpn$, $\supp(x) = \{i: x_i > 0\} \subseteq V$ is its \emph{support}. 
Conversely, if $S\subseteq V$, then $\1_S$ denotes its \emph{indicator vector}, which is 
binary and verifies $(\1_S)_i = 1 \Leftrightarrow i\in S$.
As $\supp(\1_S)=S$, 
we identify each subset $S \subseteq V$ with its indicator vector $\mathbf{1}_S \in \{0,1\}^n$.
A vector $x\in [0,1]^n$ can be interpreted as a \emph{fuzzy set} membership vector over the set of vertices $V$. 
When $x\in \{0,1\}^n$ is binary, it is the indicator vector $\1_S$ of a \emph{crisp} subset of vertices $S\subset V$. 
$S\subset V$ is an \emph{Independent Set (IS)}, $S\in \IS$,  
when no two elements of $S$ are neighbors in $A$, i.e, $A_{ij}=0$ for all $(i,j)\in S^2$,
or equivalently, $\1_S^T A \1_S = 0$.
An IS $S$ is a \emph{Maximal Independent Set (MIS)}, $S\in \MIS$, 
if it is not included in another IS. 
For any node $i \in V$, $N(i) := \{j\in V : A_{ij}=1\}$ is its \emph{neighborhood}.
We denote by $\CLO{A} := A + I$ the \emph{closed adjacency matrix}, 
where the addition of the identity matrix $I$ corresponds to the inclusion of self-loops for every vertex, 
such that $(\CLO{A})_{ii}=1$.
For a vector $x\in \Rpn$, the values $(\CLO{A}x)_i = x_i + \sum_{j\in N(i)} x_j$ are 
called the \emph{closed neighborhood sums}.

\begin{definition}
A vector $x\in \Rpn$ is \textbf{Normalizable} on a graph $A$ if and only if $\CLO{A}x > 0$. 
$\NA := \{x \in \Rpn : \CLO{A}x > \mathbf{0}\}$ denotes the set of normalizable vectors on $A$. 
\end{definition}

The non-normalizable vectors are the vectors which are identically zero on some closed neighborhoods, i.e., 
such that for some node $i$, $x_i=0$ as well as for all its neighbors.

\begin{definition}
The \textbf{Graph Normalization (GN)} map $\GN{A}: \NA \to [0, 1]^n$ 
associates to a normalizable vector $x$ the vector:
\begin{equation}
\label{eqn:gn-map}
    \GN{A}(x) := x \oslash \CLO{A}x
\end{equation}

Element-wise, this is expressed as: $(\GN{A}(x))_i = x_i / \left(x_i + \sum_{j \in N(i)} x_j\right)$.

\textbf{Iterative Graph Normalization (IGN)} is then defined by the sequence: $x^0\in \NA$; 
 $x^{k+1} = \GN{A}(x^k)$.
\end{definition}

By GN, each node is normalized in parallel 
by the sum of its own mass and the mass of its neighbors.

When the graph is the complete graph $K_n$, 
all nodes are pairwise connected, 
and the closed neighborhood sum for every node is identical: $(\CLO{A}x)_i = \sum_{j=1}^n x_j$. 
Consequently, GN corresponds exactly to simple probabilistic vector normalization, 
projecting any $x$ onto the simplex $\Delta^{n-1}$ in a single iteration. 
The \emph{binary} fixed points of the dynamics are the one-hot vectors, 
representing the MISes of $K_n$.

When the graph is the line graph $L(K_{V_1,V_2})$ corresponding to the Assignment Problem between 
two sets $V_1, V_2$, 
GN normalizes each weight $W_{ij}$, $(i,j)\in V_1\times V_2$ 
by the sum of all competing assignments in its respective row and column:
$W_{ij} + \sum_{k \in V_1 \setminus \{i\}} W_{kj} + \sum_{l \in V_2 \setminus \{j\}} W_{il}$.
We refer to this process as matrix \emph{cross-normalization}. 
GN thus replaces the alternating row/column projections of Sinkhorn-Knopp with a single normalization step over 
all conflicting assignments. 
This scheme extends beyond line graphs of assignment problems to arbitrary constraint topologies; 
while close in spirit to parallel Bregman-Dijkstra projections \cite{bauschke1996projection}, 
GN is a distinct dynamical system.

While the MISes are the only binary fixed points of the GN map, 
the system admits spurious fractional fixed points 
which can act as attractors or create plateaus that impede convergence to binary solutions.
We provide a complete characterization of the fractional attractors of this unregularized GN map,
including an exhaustive enumeration up to $n=10$ in Appendix \ref{app:fractional-atoms}.
Furthermore, a naive initialization $x^0 = w$ provides only a transient bias; 
because the GN operator is topologically driven by the adjacency matrix $\CLO{A}$, 
the unweighted dynamics remain objective-agnostic. 
To transform GN into a robust, weight-aware MWIS engine, 
we introduce two synergistic modifications: 
(i) a regularization parameter $\gamma$ that functions as a structural catalyst for binarization, 
and (ii) a scaling of the adjacency matrix that embeds the weighted objective 
directly into the system's fixed-point topology.


\begin{definition}[Weighted Regularized Graph Normalization (WRGN)]
Let $G=(A,w)$ with $w \in \Rppn$ be a weighted graph and $\gamma > 1$ a regularization parameter. 
Let $v := \sqrt{w}$. 
the \textbf{WRGN} map is the map $\GN{\WRAM}$, where $\WRAM$ is the Weighted Regularized Adjacency Matrix:
\begin{equation}
    \WRAM := \gamma \cdot \diag(v)^{-1} A \cdot \diag(v)
\end{equation}
Component-wise, this gives $(\GN{\WRAM}(x))_i = x_i / \left( x_i + \gamma \sum_{j \in N(i)} \frac{v_j}{v_i} x_j\right)$.
\end{definition}

In this formulation, $\gamma$ breaks the symmetry between self-influence and neighbor-influence. 
We show in next sections that $\gamma > 1$ 
triggers a hyperbolic transition that makes fractional solutions repulsive, 
leaving Maximal Independent Sets (MISes) as the only asymptotically stable attractors. 
Simultaneously, the bias $v$ realigns the energy landscape to prioritize high-weight vertices, 
embedding the MWIS objective directly into the dynamics regardless of initialization. 

Note that while $A$ and $I+\gamma A$ are symmetric, $\WRAM$ is \emph{not}. 
It is \emph{diagonally similar} to $I+\gamma A$ via the constant potential $v$. 
This structure implies that the pressure exerted on node $i$ by neighbor $j$ is scaled by $v_j / v_i$, 
while the reciprocal pressure is scaled by $v_i / v_j$. 
This asymmetry is what allows high-weight nodes to exert more ``normalization pressure'' 
on their lower-weight neighbors, effectively suppressing them in the competition for inclusion in the independent set.

Before studying the dynamics induced by WRGN, 
we first provide some basic properties of the WRGN map (proofs are straightforward).
For any $A, \gamma, v$, WRGN verifies: 

\textbf{i) Boundedness:} For any $x \in \NA$, $\GN{\WRAM}(x) \in [0, 1]^n$; 

\textbf{ii) Support Invariance:} The support of a vector is invariant under WRGN, 
i.e., $\supp(\GN{\WRAM}(x)) = \supp(x)$. Hence zeros are stable and do not influence the dynamics. 
We can thus always remove zero nodes from the graph and assume that $x^0$ and all the iterates have full support, 
i.e., are strictly positive vectors.

\textbf{iii) Connected Components Independence:} Because updates only depend on neighboring values, 
the dynamics between two different connected components of the graph do not influence each other. 
We can thus assume that the graph is connected. 

\textbf{iv) Scale Invariance:} $\GN{\WRAM}(x)$ is \emph{projective} by nature; for any $\alpha > 0$, $\GN{\WRAM}(\alpha x) = \GN{\WRAM}(x)$.


\section{Energy Descent and Global Convergence}

We establish that WRGN performs a Quasi-Newton descent on an energy function, 
utilizing updates that are exact Majorization-Minimization (MM) steps \cite{lee2000algorithms}. 
By leveraging the general framework of Attouch, Bolte, and Svaiter \cite{attouch2013convergence} 
for the convergence of descent-like algorithms, 
we prove systematic convergence to a normalizable fixed point of the dynamics.

\begin{definition}
Let $A$ be the adjacency matrix of a connected undirected simple graph, $v\in\Rppn$ and $\gamma > 0$.
We define the quadratic energy on $x\in \Rppn$:
\[
\ENv(x) := \frac{1}{2} x^T \CLO{A}^{\gamma, v} x - \sum v_i^2 x_i
\]
\end{definition}

For a WRGN sequence $\{x^k\}$, we define the weighted state space $y^k := v \odot x^k$, 
which is diffeomorphic to the native state space as $v > 0$. 
The energy $\ENv(x)$ in $x$ is equivalent to the energy in $y$:
$\TENv(y) = \frac{1}{2} y^T \CLO{A}^\gamma y - v^T y$ 
where $\CLO{A}^\gamma := I + \gamma A$ is the \emph{unweighted} regularized closed adjacency matrix.
We then prove in Appendix~\ref{app:proof-strict-descent}:

\begin{theorem}[Strict Energy Descent via MM Updates]
\label{thm:strict-descent}
The WRGN update corresponds to the minimization of a convex quadratic majorant 
of the energy $\TENv$ in the weighted state space $y = v \odot x$.
Consequently, the dynamics strictly decrease the energies $\TENv$ and $\ENv$ at every step
unless the system is at a fixed point.
\end{theorem}

This MM equivalence frames WRGN as a fast diagonal quasi-Newton method 
using a dynamic diagonal approximation of the inverse Hessian. 
As shown in Appendix~\ref{app:proof-precond}, the update is a preconditioned gradient descent:
\[
y^{k+1} - y^k = -\diag(y^{k+1} \oslash v) \nabla \TENv(y^k).
\]
Apart from approximating the inverse Hessian, 
the adaptive step size $\diag(y^{k+1} \oslash v)$ 
also naturally incorporates log-barriers which confine the dynamics in the unit box $[0,1]^n$.

Despite the strict decrease of the energy $\ENv$, there are two difficulties in proving 
systematic convergence to a single point.
First, the normalizable domain is an open set, with non-normalizable points belonging to its boundary.
We must make sure that the dynamics cannot fall into these singularities and remain confined inside a compact set. 
This is proven in Appendix \ref{app:proof-accu}.
Second, the WRGN map actually has non-isolated fractional fixed points which 
constitute continuous manifolds (see Appendix~\ref{app:fractional-atoms}). 
Even with vanishing steps (which directly follows from the MM theory), 
in general, the sequence could indefinitely drift along those manifolds and never settle down.
We leverage the general framework of Attouch, Bolte and Svaiter \cite{attouch2013convergence}
who have proven that if i) the energy is continuous at critical points and verifies the Kurdyka-Łojasiewicz (KL) property 
-- which characterizes the local geometry of a function, ensuring it is not "too flat" near its critical points,
and ii) the descent method ensures sufficient energy decrease and a relative error condition,
then it necessarily converges to a critical point of the energy.
This gives the following theorem (proof in Appendix \ref{app:proof-convergence}):

\begin{theorem}[Global Convergence of Iterated WRGN]
\label{thm:convergence}
For any simple undirected graph $A$, regularization $\gamma > 0$, 
bias $v \in \mathbb{R}^n_{++}$, and normalizable initialization $x^0\in\NA$, 
the WRGN sequence $\{x^k\}$ converges to a unique normalizable fixed point $x^*\in\NA$.
\end{theorem}


\section{Weighted Mass Increase, MWIS Objective Alignment and Binarization}

We now characterize the specific behavior of WRGN as a binarizing engine that 
systematically increases the weighted mass of the system, i.e., the relaxed primal objective of the MWIS problem.
We prove in Appendix~\ref{app:proof-mass-increase}:

\begin{theorem}[Weighted Mass Growth]
\label{thm:mass-increase}
Consider the WRGN sequence $\{y^k\}$ in the weighted state space $y = v \odot x$. 
The weighted mass $M_v(y) := \sum_{i=1}^n v_i y_i$ is strictly increasing: $M_v(y^{k+1}) > M_v(y^k)$ 
whenever the sequence is not at a fixed point ($y^k \ne y^{k+1}$).

In particular, when applied to the MWIS problem with $v = \sqrt{w}$, 
WRGN strictly increases the relaxed MWIS objective $\sum w_i x_i$ at each iteration.
\end{theorem}

The WRGN dynamics are thus governed by a dual Lyapunov structure: 
while the energy $\ENv$ decreases, the weighted mass, i.e. the relaxed MWIS objective, strictly increases.
Recall that the energy $\ENv(x) = Q(x) - M(x)$ 
opposes two conflicting terms: a quadratic term $Q(x) = \frac{1}{2} x^T \CLO{A}^{\gamma, v} x$ 
which is minimized when $x$ is an Independent Set, 
and the weighted mass term $M(x) = w^T x$ which is maximized on $[0,1]^n$ 
when all nodes are selected: $x=\1$.
WRGN minimization path on $\ENv$ is thus driven by a mass growing force
toward an independent solution, i.e. a local MWIS.

The WRGN map however admits non-binary fixed points in general.
The following Theorem shows that the regularization $\gamma$ acts as a symmetry-breaking force 
that makes those fractional solutions repulsive,
forcing the WRGN sequence to converge to a MIS (proof in Appendix \ref{app:proof-binary-regularization}):

\begin{theorem}[Binary Regularization]
\label{thm:binary-regularization}
For any weighted graph $G=(A, w)$ and regularization $\gamma > 1$, 
let $\mathcal{G}_\gamma = \GN{A^{\gamma, \sqrt{w}}}$ be the WRGN map. Then:
\begin{itemize}
\item{Fractional Repulsivity:} 
Every non-binary fixed point of $\mathcal{G}_\gamma$ is strictly repulsive.
\item{Binary Stability:}
A binary fixed point $x$ (with MIS $M=\text{supp}(x)$) 
is asymptotically stable iff it is \textbf{$\gamma$-stable}, i.e, verifies $\text{stab}_{\gamma, w}(M) > 1$, 
where:
\begin{equation}
\label{eqn:mis-stability}
\stab_{\gamma, w}(M) := \gamma \min_{i \notin M} \sum_{j \in N(i) \cap M} \sqrt{w_j / w_i}.
\end{equation}
\item{Optimality:}
Every MWIS of $G$ is a stable attractor of the dynamics ($\text{stab} \ge \gamma > 1$).
\item{Convergence:}
Any sequence $\{x^k\}$ initialized on a normalizable point converges to a $\gamma$-stable MIS.
The weighted mass $\sum w_i x_i^k$ strictly increases until convergence.
\end{itemize}
\end{theorem}


The threshold $\gamma=1$ represents a topological phase transition in the energy landscape. 
For $\gamma > 1$, fractional equilibria become repulsive, 
ensuring convergence to a binary MIS, a feasible solution of the combinatorial optimization problem.
Beyond ensuring proper binarization, $\gamma$ acts as a selectivity filter: 
as $\gamma \to 1^+$, the basins of attraction for sub-optimal MISes vanish, 
increasing the likelihood that the system settles into the basin of the global MWIS. 
This selectivity introduces a speed-accuracy trade-off: 
as the landscape flattens near $\gamma = 1$, 
the convergence rate slows.
The evolution of the energy landscape as $\gamma$ increases for the particular case of $K_2$ is 
studied in Appendix~\ref{app:phase-transitions}.
Unlike deterministic annealing or entropy regularization, 
which only recover discrete solutions in the theoretical zero-temperature limit, 
WRGN operates in the low-entropy regime natively. 
Binarization is the natural limit of a parameter-free, stable dynamical system.
\section{Graph Normalization as an Evolutionary Game}
\label{sec:game}

The WRGN dynamics admit a natural interpretation through Evolutionary Game Theory (EGT). 
\emph{Replicator Dynamics (RD)}  -- the canonical deterministic model of EGT -- 
describe how strategy frequencies $p_i$ in a population evolve 
based on relative fitness $f_i(p)$ \cite{taylor1978,hofbauer1998}. 
We define the \emph{simplex state} $p^k$ of the WRGN system 
by normalizing the weighted state: $p_i^k := v_i y_i^k / M_v^k$, 
where $M_v^k := \sum v_j y_j^k$ is the total weighted mass. 
We then obtain the following Theorem (proof in Appendix \ref{app:proof-RD}):

\begin{theorem}[WRGN as Nonlinear Replicator]
\label{thm:replicator}
The dynamics of the simplex state $p^k$ 
of a WRGN sequence generated by the map $\GN{\CLO{A}^{\gamma, v}}$ obey a nonlinear non-potential replicator equation:
\begin{equation}
p_i^{k+1} = p_i^k \frac{f_i(p^k)}{\bar{f}(p^k)}, \quad \text{with fitness} \quad f_i(p) = \frac{v_i}{(\CLO{A}^{\gamma}(p \oslash v))_i},
\end{equation}
where $\bar{f}(p^k) := \sum p^k_i f_i(p^k)$ is the average fitness of the population.

Moreover, $M_v^{k+1} = \bar{f}(p^k)$: 
the total weighted mass of each generation 
is exactly the average fitness of the preceding one, 
which increases monotonically throughout the dynamics.
\end{theorem}

Replicators capture the principle of natural selection: 
strategies with above-average fitness replicate, 
while those with below-average fitness decrease in frequency. 
The dynamics preserve the simplex ($\sum_i p^k_i = 1$), 
fixed points correspond to Nash equilibria, 
and asymptotically stable equilibria are Evolutionary Stable Strategies (ESS).

The WRGN Harvesting Game:
We can interpret the WRGN dynamics as a spatial competition for resources. 
Consider an infinite population of agents distributed across settlement sites with intrinsic yields $v_i = \sqrt{w_i}$. 
Let $p^k_i$ denote the population density at site $i$ at round $k$. 
Agents at site $i$ harvest resources from their local site and 
-- controlled by the interference coefficient $\gamma$ -- from adjacent sites.
The effective demand on site $i$, defined as $D^k_i = (\CLO{A}^{\gamma}(p^k \oslash v))_i$, 
represents the aggregate harvesting pressure from the closed neighborhood. 
The individual payoff (fitness) is the ratio of a site's yield to its effective demand: 
$f_i = v_i / D^k_i$. 
As agents migrate toward higher-yield locations following the Replicator Equation, 
the population's average fitness (average harvest) increases monotonically.
When $\gamma > 1$, the competitive cost of adjacency outweighs the cost of local congestion. 
In this regime, the population converges to an equilibrium where 
it occupies a set of non-adjacent sites -- an Independent Set -- 
while trying to maximize the total resource extraction $\sum w_i$, 
hence approximating the MWIS problem through natural selection.

While replicator dynamics originated in biological contexts to describe the evolution of phenotypes, 
they have since been extensively adopted in evolutionary economics to model the competition 
between firms and technologies \cite{metcalfe1994, friedman1991}. 
In an economical context, the WRGN Harvesting game can be reformulated 
as a model of capital rebalancing under structural externalities; 
for $\gamma > 1$, the dynamics act as a price discovery mechanism 
where the stable portfolio consists exclusively of non-correlated assets that maximize total yield.

Most discrete replicators rely on \emph{potential} games, 
where the fitness is defined as the gradient of a scalar potential
which increases monotonically according to Fisher's Fundamental Theorem of Natural Selection, 
and guarantees convergence. 
In contrast, for non-potential systems, 
discrete dynamics typically manifest complex behaviors 
such as chaos or limit cycles \cite{sato2003coupled}. 
WRGN is thus a rare class of globally convergent nonlinear, non-potential dynamics 
that are globally convergent and obey Fisher's principle of natural selection.

Bomze pioneered the application of \emph{linear} replicators to combinatorial optimization 
by formulating the Maximum Weight Clique (MWC) problem 
as an evolutionary game \cite{bomze1997, bomze2000, bomze2002}. 
Bomze leveraged the Motzkin-Straus theorem \cite{motzkin1965maxima}, 
which establishes a correspondence between the clique number $\omega(A)$ of a graph $A$  
and the global maximum of the quadratic form $p^TAp$ on the simplex.
Linear replicators with fitness $f(p)=Ap$ are a potential game 
on the quadratic potential $\frac{1}{2}p^TAp$, which is also the average fitness.
They thus monotonically increase the Motzkin-Straus quadratic form 
and converge to a local maximum.
However, the potential often exhibits fractional "spurious" local maxima. 
To address this, Bomze \cite{bomze1997} introduced a regularization of the adjacency matrix $A$ 
by incorporating a self-loop term: $\hat{A} = A + \frac{1}{2}I$. 
This modification ensures that the 
the asymptotically stable fixed points of the dynamics 
correspond exactly to the maximal cliques of the graph, 
thereby eliminating the fractional, 
non-clique local maxima of the original Motzkin-Straus formulation.
We drew direct inspiration from Bomze in our regularized formulation of GN.

Since an IS in $G$ is exactly a Clique in the complement graph $\overline{G}$, 
WRGN non-linear replicators and Bomze's linear replicators 
address the same combinatorial core through dual structural lenses.
However, while linear replicators correspond to first-order Mirror Descent on a potential, 
the reciprocal fitness in WRGN incorporates the local curvature of conflict, 
resulting in a nonlinear second-order-like optimization path. 
A comparison between WRGN and Bomze’s linear formulation for the \emph{unweighted case}
is provided in Table~\ref{tab:bomze-vs-wrgn}. 
This analysis reveals a fundamental complementarity: 
where the nonlinear WRGN game utilizes \emph{competition} to minimize structural conflict on the original graph, 
the linear clique game uses \emph{collaboration} to maximize mutual support on the complement graph. 
Remark the reciprocal symmetry between the two formulations of the fitness.

\begin{table}[ht]
\centering
\caption{Comparison of Replicator Dynamics for MaxMIS / MaxClique (unweighted).}
\label{tab:bomze-vs-wrgn}
\begin{tabular}{l l l}
\hline
\textbf{Feature} & \textbf{WRGN MaxIS Game} & \textbf{Bomze's MaxClique Game} \\ \hline
Optimization Target & Maximum Independent Set in $G$ & Maximum Clique in $\overline{G}$ \\
Dynamics Class & Nonlinear Replicator (2nd Order-like) & Linear Replicator (1st Order) \\
Fitness $f_i(p)$ & $1/(Bp)_i$ (Reciprocal) & $(\overline{B}p)_i$ (Linear) \\
    & with $B = I + \gamma A$ & with $\overline{B} = \frac{1}{\gamma}I + A$  \\
Potential & None (Non-potential) & Average Fitness $p^T \overline{B}^\gamma p$ \\
Other Lyapunov & Mass $\1^T y$ = Average Fitness (grows) & None \\
               & Energy $\frac{1}{2}y^T B^\gamma y - \1^T y$ (decreases) & - \\
Mechanism & Competition & Collaboration \\
\hline
\end{tabular}
\end{table}

For the weighted case (MWIS/MWC), Bomze leveraged Gibbons et al. 
weighted extension of Motzkin-Straus \cite{gibbons1997continuous}, 
who characterized the MWIS as the global \emph{minima}
on the simplex of a quadratic form $p^T G_w p$ with a particular matrix $G_w$ which incorporates the weights.
In order to fit the linear replicators \emph{maximization} framework,
Bomze uses the matrix $G'_w = \alpha J - G_w$, where $J$ is the all-ones matrix and 
$\alpha > 0$ a large enough constant to ensure that $G'_w$ is positive.
This introduces a meta-parameter $\alpha$ and breaks the natural scale invariance 
of the MWIS problem. Furthermore, in Gibbons et al.'s matrix $G_w$ 
the regularization and the weighting part are entangled.
Instead WRGN respects MWIS scale invariance and totally decouples $\gamma$ and $w$,
leading to a new weighted extension of the Motzkin-Strauss theorem:

\begin{theorem}[Weight-Tilted Simplex Motzkin-Straus]
\label{thm:motzkin-straus}
Let $G=(A, w)$ be a weighted simple undirected graph with $w_i > 0$.
Let $\Delta^w_{n-1} := \{x \in \mathbb{R}^n_+ \mid \sum \sqrt{w_i} x_i = 1\}$ be the \textbf{Weight-Tilted Simplex}. 
For any $\gamma > 1$, the local minima of the quadratic form $\mathcal{Q}(x) = x^T (I + \gamma A) x$ 
on $\Delta^w_{n-1}$ are in 1:1 correspondence with the $\gamma$-Stable Maximal Independent Sets of $G$.
The value of the $\mathcal{Q}$ at such a minimum corresponding to a MIS $M$ is $1 / \sum_{i \in M} w_i$. 
In particular, any Maximum Weight Independent Set of $G$ is $\gamma$-Stable and 
corresponds to a global minimum of $\mathcal{Q}$ on $\Delta^w_{n-1}$.
\end{theorem}

The proof is provided in Appendix \ref{app:proof-motzkin-straus}.
Note that this theorem characterizes the minima of a quadratic form 
independently of any specific optimization algorithm, such as WRGN. 
The elegance of this result lies in its geometric interpretation: 
the weight vector manifests as a deformation of the simplex, 
effectively shifting the problem's weights from the objective function ($\mathcal{Q}$) 
to the geometry of the constraint manifold ($\Delta^w_{n-1}$). 
While $\Delta_{n-1}$ is a hyperplane with normal vector $\1$, 
$\Delta^w_{n-1}$ is a hyperplane with normal vector $\sqrt{w}$. 
Geometrically, weighting the graph nodes simply manifests 
as \emph{tilting} the constraint hyperplane, 
which reconfigures the landscape of local minima to favor directions aligned with heavier weights.
On its side, the regularization factor $\gamma$ filters out the weakest MISes 
while preserving the global MWIS\footnote{Note that both Motkzin-Straus and Gibbons et al. 
theorems only characterize the global optima. Our theorem provides a richer characterization 
of all local minima.}.
It is easy to prove that for uniform weights, all MISes are $\gamma$-Stable for any $\gamma>1$.
$\Delta^w_{n-1}$ then aligns with the simplex $\Delta_{n-1}$,
the MISes and the minima of $Q$ are in 1:1 mapping and 
the global optima are the maximum \emph{size} MISes, which recovers 
Bomze's result in \cite{bomze1997}.
For the weighted case, for any graph and weights, there is a $\gamma_0$ such that 
for $\gamma > \gamma_0$ \emph{all} MISes are $\gamma$-Stable hence there 
is a 1:1 correspondence between MISes and local minima, together with alignement of 
the depth of the local minima with the weight of the MISes.

\section{Experimental Results}
\label{sec:experiments}

GN can be leveraged as an MWIS solver using either random initializations 
or as a fast, accurate binarization engine for methods that excel at providing fractional solutions. 
The state-of-the-art method in this domain is the Bregman-Sinkhorn (BS) MWIS solver by Haller and Savchynskyy \cite{haller2024}. 
The authors solve the clique-cover LP relaxation of the MWIS problem 
via entropy-regularized dual coordinate descent, 
which utilizes a Bregman-Sinkhorn projection on the clique constraints at its core. 
While BS is highly efficient at reaching a low \emph{relaxed} duality gap
-- frequently in under 1s for graphs with tens of thousands of nodes -- 
it relies on randomized greedy heuristics and optimal recombination to round 
fractional solutions into feasible integer sets 
(i.e., actual 'hard' solutions to the discrete combinatorial problem). 
In practice, reaching low integral gaps takes orders of magnitude longer 
than reaching low fractional gaps. 
We demonstrate that GN can be efficiently employed to bridge this fractional-to-integral gap.

\textbf{Tracking the global minimum by $\gamma$-pursuit.}
As discussed previously, in the limit $\gamma \to 0^+$, 
the GN energy is convex and possesses a unique fractional global minimum. 
As $\gamma$ increases, the energy landscape deforms until the most stable MIS becomes $\gamma$-stable. 
This eventually leads to a regime where only MISes are stable, which is guaranteed for $\gamma > 1$. 
We refer to this graduated non-convexity tracking algorithm as \emph{$\gamma$-pursuit}. 
A PyTorch implementation of a Graph Normalization layer incorporating this tracking logic 
is provided in Appendix~\ref{sec:appendix-code}. 
In all experiments, we utilized $1,000$ GN iterations 
with a linear increase of $\gamma$ from $\gamma_0=0.9$ to $\gamma_1=1.5$. 
These parameters ensured convergence to a valid MIS (after rounding) for all tested graphs, 
demonstrating robust and rapid convergence even for the largest instances in the dataset.

\textbf{Randomized start and warm start.}
We compared two distinct initialization strategies of WRGN sequences: 
random sampling on the simplex and warm-starting from fractional solutions computed via BS. 
For the warm-start configuration, we modified the BS source code provided at 
\url{https://github.com/vislearn/libmpopt} \cite{haller2024} 
to execute the BS fractional optimization scheme in isolation, 
bypassing the default heuristic binarization and rounding phases. 
We ran this modified solver until the \emph{fractional} duality gap fell below $0.1\%$, 
a threshold typically reached in under one second for graphs with $10,000$ nodes. 
At this stage, the \emph{integral} gap often remains near $100\%$, 
and reducing it below $1\%$ requires several minutes of additional computation.
To ensure diversity within our warm-start batches, 
we generated a variety of fractional solutions by 
randomizing the clique optimization order 
and jittering initial temperatures uniformaly in the interval $[100, 110]$.
The temperature drop factor was set to $0.99$.

\textbf{Datasets. }
We benchmarked these approaches on the $59$ large-scale real-world MWIS instances provided by \cite{haller2024}. 
This dataset includes the Amazon Vehicle Routing (AVR) benchmark, 
where nodes represent potential routes and edges represent conflicts between them, 
and the Meta-Segmentation for Cell Detection (MSCD) dataset, 
in which MWIS is used to select non-overlapping cell segmentations. 
These datasets are particularly challenging as they scale up to 882K nodes and 344M edges 
(see Table \ref{tab:results-summary}),
and contain the largest real-world MWIS instances publicly available.

\textbf{Hardware.}
All experiments were conducted on an Apple M4 Pro CPU with 64GB of unified memory. 
We hit memory limitations though, both with the \texttt{libmpopt} Bregman-Sinkhorn code and our WRGN code.
BS memory usage scales linearly with the number of cliques in the clique cover 
and could not run on 4/59 graphs with the largest number of cliques.
WRGN memory usage scale linearly with the number of edges in the graph,
and although we could run on all graphs (except the 4 for which BS did not run), 
memory limitations triggered page swaps for graphs exceeding 10M edges. 
For these instances, we limited testing to a single warm start instead of $16$ runs 
for other graphs; consequently, these runtimes do not reflect the 
intrinsic efficiency of the WRGN operator.

\textbf{Results. }
Table \ref{tab:results-summary} summarizes the performance across 55/59 instances. 
The full results are reported to appendix \ref{app:exp-results}.
$\texttt{BS Time}$ is the time of a BS run as reported in \cite{haller2024}; $\texttt{GN Time}$ is the time of a GN run.
$\texttt{Gap}$ denotes the relative difference between the GN solutions and the solutions reported in \cite{haller2024}, 
which were obtained with 1h runs of BS, and constitute the best known solutions for each problem to date, 
beating multiple competing algorithms including leading commercial software such as Gurobi.
$\texttt{E[Gap]}$ and $\texttt{Best Gap}$ are resp. the average and the minimum gap found by GN 
for each type of start (random or warm).

\begin{table}[h]
\centering
\caption{Summary of MWIS Performance on Representative Instances.}
\label{tab:results-summary}
\begin{tabular}{@{}lcccccc@{}}
\hline
\textbf{Instances} & \textbf{Edges} & \textbf{BS Time} & \textbf{GN Time} & 
\multicolumn{2}{c}{\textbf{Best Gap}} \\ 
& & & & \textbf{Random Start} & \textbf{Warm Start} \\ \hline
AVR\_024 / 034 & 3K / 23K & $<1$s & $0.4$s / $0.1$s & $3.10$\% / $22.7$\% & $0.00$\% / $0.00$\% \\
AVR\_023 / 027 & 127K / 2K & $<1$s & $1.4$s / $0.2$s & $21.09$\% / $1.74$\% & $0.03$\% / $0.06$\% \\
MSCD\_* & 67K-198K & $<1$s & $0.9-1.9$s & $0.77-5.79$\% & $0.41-1.34$\% \\
AVR\_026 / 028 & 604K / 516K & $<1$s & $5$s / $4.6$s & $4.71$\% / $3.32$\% & $0.39$\% / $0.54$\% \\
AVR\_036 & 1.5M & $18$s & $12.7$s & $25.34$\% & $23.09$\% \\
AVR\_033 & 4M & $<1$s & $27$s & 10.33\% & $2.76$\% \\
AVR\_009 & 43M & $10$s & $801$s$^*$ & - & $5.70$\% \\
AVR\_002 & 344M & $297$s & $3751$s$^*$ & - & $8.93$\% \\
\hline
\end{tabular}
\end{table}

Those results yield several key insights:
Warm-starting from a BS fractional solution provides a significantly better initial basin. 
While random starts frequently result in gaps exceeding $5-10\%$, warm-starting brings the solution within $1\%$ of the best-known 
integer solution for almost all instances under $1M$ edges.
In particular, GN finds the best known solution on AVR\_024 and AVR\_034, 
and solutions below $0.1\%$ gap on AVR\_023 and AVR\_027, in less than $1$s. 
Some instance, such as AVR\_036, are particularly hard, both for BS ad GN, which is apparent in the time taken by BS to 
reach a fractional gap of $0.1$\% and by the poor quality of the final solution found by GN, even compared to larger graphs.

\section{Conclusion}
We have introduced Graph Normalization (GN), 
a principled dynamical system that provides a fast, 
differentiable engine for the Maximum Weight Independent Set problem. 
By establishing an equivalence between GN and nonlinear Replicator Dynamics, 
we have provided a rigorous game-theoretic foundation for binarization 
that avoids the sensitive hyperparameter tuning required by deterministic annealing. 
Our theoretical contributions -- including the Weight-Tilted Simplex Motzkin-Straus theorem -- 
characterize the landscape of the MWIS problem as a geometric deformation of the constraint manifold, 
where weights manifest as the "tilt" of a hyperplane and binarization emerges as a natural topological phase transition.
Empirically, GN serves as a high-speed binarization bridge for state-of-the-art relaxed solvers, 
reaching within a 1\% gap of best-known solutions for production-scale graphs in seconds. 
Beyond combinatorial optimization, 
the differentiability of the GN layer allows for its seamless integration into deep learning pipelines. 
This opens new avenues for architectures that require "hard" decisions under structural constraints, 
such as sparse Mixture-of-Experts, dynamic network pruning, and end-to-end learnable structured attention. 
Future work will investigate the extension of GN to hypergraph constraints 
and the potential for a continuous-time formulation of the dynamics 
to further improve convergence speed on massive-scale relational data.


\bibliographystyle{plain}
\bibliography{GraphNormalization}
\appendix
\section{Fixed Points of the Unregularized GN Map}
\label{app:fractional-atoms}

A fixed point $x$ of the GN map defined by Equation~\ref{eqn:gn-map} 
is a normalizable vector $x\in\NA$ which verifies componentwise:
\begin{equation}
\label{eqn:fixed-point}
x_i = x_i / ((A+I)x)_i.    
\end{equation}

Hence $x_i$ is fixed if and only if either $x_i = 0$ or $((A+I)x)_i = 1$.

Let $x\in\NA$ be a fixed point and 
$C_1, C_2, \dots C_p$ the connected components of its support $\supp(x)$.
Assume $i\in\supp(x)$ belongs to the component $C_k$. 
As its closed neighborhood sum $((A+I)x)_i = x_i + \sum_{j\in N(i)} x_j = x_i + \sum_{j\in N(i)\cap C_k} x_j$,
the restriction $x_{C_k}$ of $x$ to $C_k$ is a fixed point of the subgraph $A_{C_k}$ induced by $C_k$.
We call a fixed point with \emph{full support} on a \emph{connected graph} an \textbf{Atom} of Graph Normalization (GNA).
Any fixed point thus uniquely decomposes into i) zeros,
and ii) atoms over the connected components of its support.

For example, the following fixed point is made up of 3 atoms highlighted in grey: 
\begin{figure}[H]
\centering
\begin{tikzpicture}[
node distance=2cm, 
every node/.style={draw, circle, fill=white, minimum size=1.2cm}, 
every edge/.style={draw, ->} 
]
\node[shape=circle,draw=black,scale=0.5] (a) at (0,0) {0};
\node[shape=circle,draw=black,scale=0.5,fill=gray!20, thick, font=\bfseries] (b) at (1,0) {1};
\node[shape=circle,draw=black,scale=0.5] (c) at (1,1) {0};
\node[shape=circle,draw=black,scale=0.5,fill=gray!20,thick, font=\bfseries] (d) at (0,1) {1};
\node[shape=circle,draw=black,scale=0.5] (e) at (2,1) {0};
\node[shape=circle,draw=black,scale=0.5,fill=gray!20,thick, font=\bfseries] (f) at (2,2) {1/4};
\node[shape=circle,draw=black,scale=0.5,fill=gray!20,thick, font=\bfseries] (g) at (1,2) {3/4};
\draw (a) -- (b);
\draw (b) -- (c);
\draw (c) -- (d);
\draw (d) -- (a);
\draw (b) -- (e);
\draw (e) -- (f);
\draw[thick] (f) -- (g);
\draw (g) -- (c);
\draw (c) -- (e);
\end{tikzpicture}
\end{figure}

It is not binary (values $1/4$ and $3/4$), 
but it has two positive binary components which correspond to the smallest 
GN atom: the value $1$ over the complete graph $K_1$ with only 1 node and no edge.
The MISes of a graph are precisely the fixed points which decompose into zeros and $K_1$ atoms 
-- which are non adjacent by definition of the decomposition by connectivity.

Atoms constitute the building blocks of the general fixed points of GN.
An atom $x$ over a connected graph $A$ is a solution of:
\begin{equation}
\label{eqn:atom}
(A+I)x = \1. \quad \textrm{s.t.} \quad x>0.
\end{equation}
Hence, all the closed neighborhood sums of an atom 
-- the sum of the value of a node plus the values of its neighbors -- are equal to $1$, i.e., normalized.
Note that for a generic fixed point, with zeros, the closed neighborhood sums at zero nodes 
are not normalized in general. Consider for example $x=(1, 0, 1)$ 
on the path graph 'o-o-o' with adjacency $A=\begin{pmatrix} 0 & 1 & 0 \\ 1 & 0 & 1 \\ 0 & 1 & 0\end{pmatrix}$.
The closed neighborhood sum of the zero node is $2$.
Note that as a fixed point must be normalizable, none of its closed neighborhood sums can be null, 
i.e., any zero must have at least one nonzero neighbor.

For a given connected graph $A$, we call the set $\mathcal{S}(A)$ of the GN atoms on $A$, 
i.e., the solutions of Equation~\ref{eqn:atom}, the \textbf{Atomic Spectrum} of $A$.
Let $B:=A+I$.
The atomic spectrum of a connected graph $A$ can be either:

\textbf{1. A single point (discrete):} $\mathcal{S}(A)=\{x\}$ when $\det(B)\ne 0$ and $x=B^{-1}\1 > 0$.

In that case, because $B^{-1}=\frac{1}{\det(B)}\adj(B)$ where $\adj(B)$ is the adjoint of $B$ and $B$ is binary, 
$\det(B)$ and $\adj(B)$ both have integer values,
hence all values of $B^{-1}\1$ are fractions proportional to $1/|\det(B)|$.
Figure \ref{fig:discrete-atoms} provides some examples.

\begin{figure}[H]
\centering
\caption{Examples of Discrete GN Atoms.}
\label{fig:discrete-atoms}
\setlength{\tabcolsep}{3pt}
\begin{tabular}{@{}cccc@{}}
\includegraphics[align=c, width=0.25\textwidth]{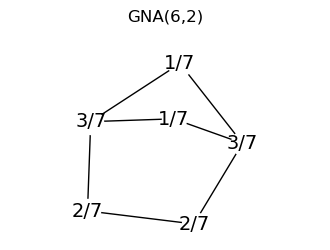} & 
\includegraphics[align=c, width=0.25\textwidth]{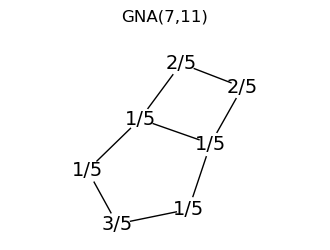} & 
\includegraphics[align=c, width=0.25\textwidth]{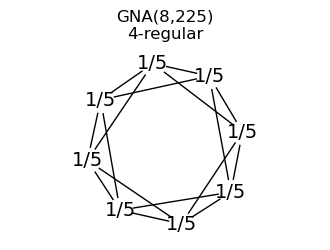} & 
\includegraphics[align=c, width=0.25\textwidth]{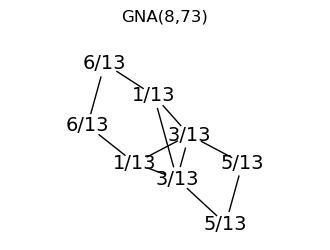} 
\end{tabular}
\end{figure}

\textbf{2. A continuous manifold:} when $\det(B)=0$ and $\exists x_0>0 : Bx_0 = 1$.

In that case $\mathcal{S}(A)=(x_0 + \ker(B)) \cap (0,1)^n$ 
is the intersection of an affine space with the open unit cube.
The manifold of fixed points is connected and has dimension $\dim(\ker(B))$.
Figure~\ref{fig:continuous-atoms} provides some examples. The dimension of the manifold is indicated by \texttt{Dim}.
The values on the nodes contain \texttt{Dim} variables \texttt{a, b, \dots}.
Any value of these variables which generate values in $(0,1)$ for all the nodes 
defines a valid atom.

\begin{figure}[H]
\centering
\caption{Examples of Continuous GN Atoms.}
\label{fig:continuous-atoms}
\setlength{\tabcolsep}{3pt}
\begin{tabular}{@{}cccc@{}}
\includegraphics[align=c, width=0.2\textwidth]{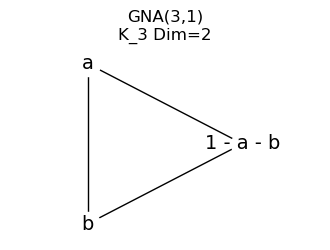} & 
\includegraphics[align=c, width=0.2\textwidth]{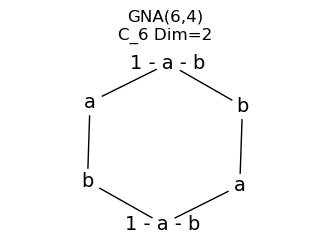} & 
\includegraphics[align=c, width=0.2\textwidth]{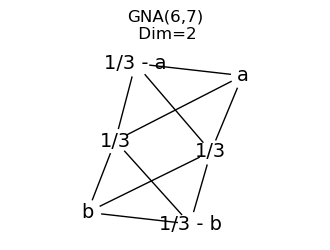} & 
\includegraphics[align=c, width=0.2\textwidth]{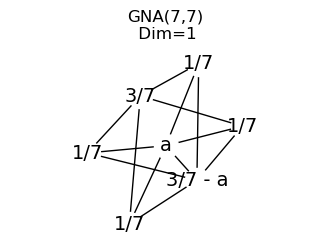} 
\end{tabular}
\end{figure}

\textbf{3. Empty:} $\mathcal{F}(A) = \emptyset$ otherwise.

Most GN atomic spectra are \emph{empty}. 
We say that a graph with a non empty spectrum is \textbf{Atomic}.

\textbf{Regular graphs} are always atomic: 
indeed if $A$ is $d$-regular then the uniform vector $x_0 = \frac{1}{d+1}\1$ 
is always a solution of $Bx_0 = \1$\footnote{One can also prove that regular graphs are the only 
connected graphs which are normalized by a uniform vector.}.
Among regular graphs, some have a discrete spectrum (e.g. the $4$-cycle $C_4$), 
i.e., are only normalized for uniform weights, but others have a continuous spectrum.
The condition is $\det(B)=0$, i.e. $-1$ is an eigenvalue of the adjacency matrix $A$.
There is no trivial condition for when this spectral condition happens in general, however
one can identify some particular \emph{families} of regular graphs with continuous spectra:

\begin{itemize}
\item \textbf{Circulant graphs:} $C_n(S)$ for which $-1$ is an eigenvalue of $A$. This includes:
\begin{itemize}
\item Complete graphs for $n\ge 2$: $\mathcal{S}(K_n)=\{x>0:\sum x_i = 1\}$ is a continuous spectrum of dimension $n-1$.
\item $3n$-cycles: $\mathcal{S}(C_{3n})=\{\underbrace{(a,b,c, \dots, a,b,c)}_{n~\text{times}}): a, b, c \in (0,1); a+b+c=1\}$ is a continuous spectrum of dimension $2$ (note that $C_3 = K_3$).
\item $C_n(1,2)$ for $n=6,10,\dots$
\item $C_n(1,3)$ for $n=8,10,\dots$
\end{itemize}
\item \textbf{Hamming graphs:} $H(d,q) = K_q^{\square d}$ (Cartesian product of $d$ copies of $K_q$, vertex-transitive and $d(q-1)$-regular); eigenvalues are $d(q-1)-qi$ with multiplicity $\binom{d}{i}(q-1)^i$, containing $-1$ for infinitely many $(d,q)$ pairs. In particular any \textbf{Hypercube} $Q_n=K_2^{\square d}$ for $d$ odd (eigenvalue $-1$ has multiplicity $\binom{d}{(d+1)/2}$)
\item \textbf{Generalized Petersen graphs:} (3-regular) e.g., the Desargues graph $GP(10,3)$ for which 
$\spec(A)=\{3,2^4,1^5,0^6,(-1)^4,(-2)^5\}$
\item \textbf{Odd graphs:} $O_k$ ($k$-regular, distance-regular, vertex-transitive) for $k\geq 4$; e.g., $O_4$ for which $\spec(A)=\{4,2^{14},(-1)^{14},(-3)^6\}$
\item \textbf{Cayley graphs:} Various constructions on non-abelian groups (dihedral, symmetric groups) with generating sets giving eigenvalue $-1$
\end{itemize}

We enumerated the $11,716,571$ connected graphs up to $10$ nodes 
-- without isomorphic duplicate and in canonical order -- using \emph{nauty} \cite{mckay1981practical}, 
and tested if they were atomic. 
We classified atomic graphs into four classes: $(\textrm{Irregular}, \textrm{Regular}) \times (\textrm{Discrete}, \textrm{Continuous})$.
The counts for each class are provided in table \ref{tab:atoms_10}.
One sees that the overall density of atomic graphs quickly decreases with $n$ and that 
irregular atomic graphs are much more frequent than regular ones.
Discrete and continuous spectra seem to appear approximately in the same proportions.

\begin{table}[ht!]
    \centering
    \caption{Enumeration of GN Atomic Graphs up to $10$ Nodes.}
    \label{tab:atoms_10}
    \begin{tabular}{|c|c|c|c|c|c|c|c|}
        \hline
        n & Connected & \multicolumn{4}{c|}{Atomic Graphs} & Total & Density \\ 
          & Graphs & \multicolumn{2}{c|}{Irregular} &  \multicolumn{2}{c|}{Regular} &  &  \\ 
          &  & Discrete & Continuous & Discrete & Continuous &  &  \\ \hline
        1 & 1 & 0 & 0 & 1 & 0 & 1 & 100\% \\
        2 & 1 & 0 & 0 & 0 & 1 & 1 & 100\% \\ 
        3 & 2 & 0 & 0 & 0 & 1 & 1 & 50\% \\
        4 & 6 & 0 & 0 & 1 & 1 & 2 & 33\% \\
        5 & 21 & 1 & 1 & 1 & 1 & 4 & 19\% \\ 
        6 & 112 & 2 & 6 & 3 & 2 & 13 & 11.6\% \\
        7 & 853 & 19 & 21 & 2 & 2 & 44 & 5.1\% \\ 
        8 & 11,117 & 96 & 134 & 9 & 8 & 230 & 2.1\% \\ 
        9 & 261,080 & 1,202 & 1,091 & 12 & 10 & 2,315 & 0.89\% \\
        10 & 11,716,571 & 25,799 & 17,782 & 75 & 92 & 43,748 & 0.37\% \\ \hline
    \end{tabular}
\end{table}

\section{Proof of Theorem~\ref{thm:strict-descent} (Strict Energy Descent)}
\label{app:proof-strict-descent}

We work in the weighted state space $y = v\odot x$. 
Define $B_v := I + \gamma \diag(v)^{-1}A\diag(v)$ and $B:=I+\gamma A$. 
One verifies by substitution that the WRGN update $x^{k+1}_i = x^k_i / (B_v x^k)_i$
provides the update for the state $y$: $y^{k+1}_i = v_i y_i^k / (B y^k)_i$
and that the energy $\ENv(x) =  \frac{1}{2} x^T B_v x - \sum v_i^2 x_i $ 
is equal to the energy in $y$:
\[
\TENv(y) := \frac{1}{2} y^T B y - \sum v_i y_i
\]

Following the Majorization-Minimization (MM) framework \cite{lee2000algorithms}, 
we define an auxiliary function $G(y, y^k)$ 
that serves as a global upper bound for $\TENv(y)$:
\[
G(y, y^k) := \frac{1}{2} \sum_{i} \frac{(B y^k)_i}{y_i^k} y_i^2 - \sum_{i} v_i y_i 
\]

This surrogate function satisfies the two necessary MM conditions:

\begin{enumerate}
\item Equality at the Tangent Point: By substitution, $G(y^k, y^k) = \frac{1}{2} (y^k)^T B y^k - v^T y^k = \TENv(y^k)$.
\item Global Dominance: Because $B$ is a \emph{symmetric} non-negative matrix, 
the diagonal majorization lemma \cite{lee2000algorithms} ensures 
$y^T B y \le \sum_{i} \frac{(By^k)_i}{y_i^k} y_i^2$ thus $\TENv(y) \le G(y, y^k)$ for all $y \ge 0$.
\end{enumerate}

The next iterate $y^{k+1}$ is obtained by minimizing the majorant.
Since $G$ is a separable quadratic function with a diagonal Hessian 
$\nabla^2_{yy} G = \text{diag}\bigl( \frac{(B y^k)_i}{y^k_i} \bigr)$, 
it is strictly convex provided $y^k > 0$, 
which can always be assumed as zeros are fixed and can be ignored in the dynamics.

Setting the gradient $\nabla_y G = 0$ yields the unique global minimizer:
\[
y_i^{k+1} = y_i^k \frac{v_i}{(B y^k)_i}
\]

We thus recover the WRGN update rule in weighted state space, 
which is equivalent to the update rule in unweighted state space and 
shows that WRGN is an \emph{exact} MM algorithm.

By the fundamental properties of MM algorithms:
\[ \TENv(y^{k+1}) \le G(y^{k+1}, y^k) \le G(y^k, y^k) = \TENv(y^k) \]

The first inequality is given by global dominance. 
The second inequality is strict ($G(y^{k+1}, y^k) < G(y^k, y^k)$) 
iff $y^{k+1} \neq y^k$ due to the strict convexity of the surrogate.

\section{WRGN Preconditioned Gradient Descent Form}
\label{app:proof-precond}
The gradient of the energy $\TENv$ is:
\[
\nabla \TENv(y^k) = \CLO{A}^\gamma y^k - v.
\]

The WRGN update rule for the state $y$ is:
\[
y_i^{k+1} = \frac{v_i y_i^k}{(\CLO{A}^\gamma y^k)_i}.
\]

Substituting $(\CLO{A}^\gamma y^k)_i = v_i + \nabla \TENv(y^k)_i$:
\[
y_i^{k+1} (v_i + \nabla \TENv(y^k)_i) = v_i y_i^k.
\]

Expanding and isolating the step $y_i^{k+1} - y_i^k$:
\[
y_i^{k+1} - y_i^k = -\frac{y_i^{k+1}}{v_i} \nabla \TENv(y^k)_i.
\]

In vector form, this is the preconditioned gradient descent update:
\[
y^{k+1} - y^k = -\text{diag}(y^{k+1} \oslash v) \nabla \TENv(y^k).
\]

\section{WRGN Cannot Accumulate on Non-Normalizable Points}
\label{app:proof-accu}

\begin{proposition}
\label{prop:no_accu_on_non_normalizable}
For any graph $A$, $v\in\Rppn$, $\gamma > 0$, and $x^0 \in \NA$, 
the WRGN sequence $\{x^k\}_{k \geq 0}$ cannot accumulate on a non-normalizable point. 

Consequently, there exists some $\delta > 0$ such that for sufficiently large $k$, 
the sequence is confined to the compact set:
\[K_\delta := \{x \in [0, 1]^n \mid d(x, \partial \NA) \ge \delta\}.\]
\end{proposition}

\begin{proof}
Suppose there exists an accumulation point $x^* \in \partial \NA$. 
By definition of the normalizable domain for the WRGN map, 
there exists at least one node $i$ such that its weighted neighborhood sum is zero: 
$(\WRAM x^*)_i = 0$.
Let $U = \{x \in \NA \mid (\WRAM x)_i < 1/2\}$ be an open set containing $x^*$. 
According to the WRGN update rule, when $x^k \in U$, 
the iterate $x_i^{k+1}$ satisfies:
\[
x_i^{k+1} = \frac{x_i^k}{(\WRAM x^k)_i} > 2x_i^k.
\]

We now show that for any $L \in \mathbb{N}$, 
there exists a contiguous block of $L$ iterates $\{x^k, \dots, x^{k+L}\}$ contained in $U$. 

Since the map $x \mapsto (\WRAM x)_i$ is continuous and $(\WRAM x^*)_i = 0$, 
$U$ is an open neighborhood of $x^*$. 
There exists an $\epsilon > 0$ such that the open ball $\mathcal{B}(x^*, \epsilon) \subset U$.

As established, WRGN is an exact MM algorithm minimizing the energy $\ENv$. 
Given that $\ENv$ is bounded below on the invariant set $[0,1]^n$, 
the vanishing step size $\|x^{k+1} - x^k\|_2 \to 0$ is a classical result of MM theory \cite{hunter2004tutorial}.
Since step sizes vanish, for any $L$, there exists $K_L$ such that for all $k > K_L$, $\|x^{k+1} - x^k\|_2 < \frac{\epsilon}{2L}$. 
Because $x^*$ is an accumulation point, there exist infinitely many indices $k > K_L$ such that $\|x^k - x^*\|_2 < \epsilon/2$. 
For any $m \in \{1, \dots, L\}$, the triangle inequality yields:
\[
\|x^{k+m} - x^*\|_2 \le \|x^k - x^*\|_2 + \sum_{j=0}^{m-1} \|x^{k+j+1} - x^{k+j}\|_2 < \frac{\epsilon}{2} + m\frac{\epsilon}{2L} \le \epsilon
\]
Thus, $x^{k+m} \in \mathcal{B}(x^*, \epsilon) \subset U$ for all $m \in \{0, \dots, L\}$.

During this contiguous block of length $L$,
the growth property holds at every step, yielding $x_i^{k+L} > 2^L x_i^k$. 
Since $x^0$ has full support and the WRGN map is support-invariant, $x_i^k > 0$ for all finite $k$. 
For any $k > K_L$ that is sufficiently close to $x^*$, 
$x_i^k$ is a fixed positive value. 
We can choose an $L$ large enough such that $2^L > 1/x_i^k$, ensuring $x_i^{k+L} > 1$.
This contradicts the fact that the WRGN sequence is bounded within $[0, 1]^n$. 
Therefore, $x^*$ cannot be an accumulation point. 
 
Since the sequence $\{x^k\}$ is bounded in $[0,1]^n$ and cannot accumulate on the boundary $\partial \NA$, 
it must be confined to the compact set $K_\delta$ for some $\delta > 0$.
\end{proof}

\begin{corollary}[Accumulation on Fixed Points]\label{cor:accu-fixed-points}
Any WRGN sequence accumulates on the (normalizable) fixed points of the map.
\end{corollary}
\begin{proof}
By Proposition~\ref{prop:no_accu_on_non_normalizable}, 
the sequence is confined to a compact subset of $\NA$ for sufficiently large $k$. 
Therefore, the sequence must have at least one accumulation point $x^*$. 
The graph normalization map $\GN{\WRAM}$ is continuous on $\NA$. 
Let $\{x^{k_j}\}$ be a subsequence converging to $x^*$. 
Since the step sizes vanish ($\|x^{k+1} - x^k\|_2 \to 0$ ), we have:
\[
\GN{\WRAM}(x^*) = \GN{\WRAM}(\lim_{j \to \infty} x^{k_j}) = \lim_{j \to \infty} \GN{\WRAM}(x^{k_j}) = \lim_{j \to \infty} x^{k_j+1}.
\]

Because $\|x^{k_j+1} - x^{k_j}\|_2 \to 0$, it follows that $\lim x^{k_j+1} = \lim x^{k_j} = x^*$. 
Thus, $\GN{\WRAM}(x^*) = x^*$, showing that every accumulation point is a fixed point, which is normalizable by definition.
\end{proof}

\section{Proof of Theorem~\ref{thm:convergence} (Global Convergence)}
\label{app:proof-convergence}

We leverage the following Theorem by Attouch, Bolte and Svaiter \cite[p.\,12]{attouch2013convergence}: 

\textbf{Theorem 2.9} \emph{(Convergence to a critical point). Let $f:\mathbb{R}^n\to \mathbb{R}\cup\{+\infty\}$
be a proper lower semicontinuous function. Consider a sequence $(x^k)_{k\in \mathbb{N}}$ 
that satisfies H1, H2, and H3.
If $f$ has the Kurdyka-\L ojasiewicz property at the cluster point $\tilde{x}$ specified in H3 then the
sequence $(x^k)_{k\in \mathbb{N}}$ converges to $\bar{x} = \tilde{x}$ as $k$ goes to infinity, and $\bar{x}$ is a critical point of $f$.}

where H1, H2 and H3 are defined on p.\,8 in \cite{attouch2013convergence} and reproduced below.

Let $\{y^k\}$ be a WRGN sequence in weighted state space.
Let $\delta > 0$. 
By Proposition~\ref{prop:no_accu_on_non_normalizable}, 
there exists an integer $K$ such that the tail sequence $x=\{y^{k+K}\}_{k \in \mathbb{N}}$ 
is confined to the compact set $K_\delta$. 
We define the extended-value objective on $\mathbb{R}^n$:
\begin{equation}
\label{eqn:extended-objective}
f(z) = \begin{cases} 
    \TENv(z) & z \in K_\delta \\ 
    +\infty & \text{otherwise} \end{cases}.
\end{equation}

$f$ is proper, lower semicontinuous, and, 
being a composition of semi-algebraic terms, 
satisfies the Kurdyka-Łojasiewicz property everywhere. 

We verify that the tail sequence $x$ satisfies conditions H1-H3 of Theorem~2.9 \cite[p.\,12]{attouch2013convergence}:

\textbf{H1.} (Sufficient decrease condition). $a>0$ exists such that for each $k\in\mathbb{N}$,
\[
f(x^{k+1}) + a\|x^{k+1}-x^{k}\|^2 \le f(x^k).
\]

\textbf{Verification:}
By the MM construction, $x^{k+1}$ is the minimizer of a separable quadratic surrogate function $G(x|x^k)$. 
Specifically, for $x \in K_\delta$, the second derivatives of the surrogate components are bounded 
below by a constant $a_{min} > 0$ related to the minimum weighted neighborhood sum. 
By the exact second-order Taylor expansion of $G$ around its minimizer $x^{k+1}$, 
noting that $\nabla G(x^{k+1}|x^k) = 0$ and $G(x^k|x^k) = f(x^k)$, we have:
\[
f(x^k) = G(x^k|x^k) \ge G(x^{k+1}|x^k) + \frac{a_{min}}{2}\|x^{k+1} - x^k\|^2.
\]
Since $f(x^{k+1}) \le G(x^{k+1}|x^k)$ by the majorization property, H1 is satisfied with $a = a_{min}/2$.

\textbf{H2.} (Relative error condition). $b>0$ exists such that for each $k\in\mathbb{N}$, there exists $w^{k+1} \in \partial f(x^{k+1})$ 
such that
\[
\|w^{k+1}\| \le b\|x^{k+1}-x^{k}\|.
\]
[where $\partial f$ is the limiting subdifferential of $f$]

\textbf{Verification:}
Since $x^{k+1}$ is an interior point of $K_\delta$ for sufficiently large $k$ 
we choose $w^{k+1} = \nabla f(x^{k+1}) \in \partial f(x^{k+1})$.
As $x^{k+1}$ minimizes the surrogate $G(x|x^k)$, 
we have $\nabla G(x^{k+1}|x^k) = 0$, and thus:
\[
\|w^{k+1}\| = \|\nabla f(x^{k+1}) - \mathbf{0}\| = \|\nabla f(x^{k+1}) - \nabla G(x^{k+1}|x^k)\|.
\]

Since $\nabla f(x^k) = \nabla G(x^k|x^k)$ by the tangency property of the MM surrogate, 
the difference $\|\nabla f(x^{k+1}) - \nabla G(x^{k+1}|x^k)\|$ is bounded by $b\|x^{k+1} - x^k\|$ 
for some $b>0$ due to the Lipschitz continuity of the gradients on $K_\delta$.
This satisfies H2.

\textbf{H3.} (Continuity condition). There exists a subsequence $(x^{k_j})_{j\in\mathbb{N}}$ and $\tilde{x}$ such that
\[
x^{k_j} \to \tilde{x} \quad \textrm{and} \quad f(x^{k_j}) \to f(\tilde{x}) \quad \textrm{as} \quad j\to \infty.
\]

\textbf{Verification:}
Proposition \ref{cor:accu-fixed-points} shows that $x$ 
has at least one accumulation point $\tilde{x} \in K_\delta$ and a subsequence $x^{k_j} \to \tilde{x}$. 
Since $\TENv$ is a smooth rational function on $K_\delta$, it is continuous at $\tilde{x}$. 
Thus, $f(x^{k_j}) \to f(\tilde{x})$ as $j \to \infty$, satisfying H3.

\textbf{Application of Theorem 2.9}
All conditions H1--H3 are satisfied hence the sequence $\{x^k\}$ converges to $\bar{x} = \tilde{x}$.

\section{Proof of Theorem~\ref{thm:mass-increase} (WRGN Increases the Weighted Mass)}
\label{app:proof-mass-increase}

\begin{lemma}
\label{lem:weighted-conjecture-81}
Let $S$ be a symmetric matrix with $S_{ii} = 1$ and $S_{ij} \ge 0$. Let $v \in \mathbb{R}^n_{++}$ be a vector of positive weights. 
Let $y \in \mathbb{R}^n_{++}$ be the result of a WRGN update step in weighted state space, 
i.e., there exists $z > 0$ such that $y_i = \frac{v_i z_i}{(Sz)_i}$. Then:
\begin{equation}
y^T S y \le \sum_{i=1}^n v_i y_i
\end{equation}
with equality if and only if $y$ is a fixed point of the dynamics, satisfying $(Sy)_i = v_i$ for all $i \in \text{supp}(y)$.
\end{lemma}

\begin{proof}
Let $d_i = (Sz)_i$. By the definition of the WRGN update, we have $z_i = \frac{y_i d_i}{v_i}$. 
Substituting this into the expansion of $d_i = (Sz)_i$:
\[
d_i = \sum_{j=1}^n S_{ij} z_j = \sum_{j=1}^n S_{ij} \frac{y_j d_j}{v_j} \implies 1 = \sum_{j=1}^n S_{ij} \frac{y_j d_j}{v_j d_i}.
\]

Define the weighted difference $\Delta_v(y) = \sum v_i y_i - y^T S y$. Expanding the quadratic form and using $S_{ii}=1$:
\[
\Delta_v(y) = \sum_i y_i (v_i - y_i) - \sum_{i \neq j} S_{ij} y_i y_j.
\]

Substituting the identity $1 = \sum_j S_{ij} \frac{y_j d_j}{v_j d_i}$ into the $v_i$ term:
\begin{eqnarray*}
\Delta_v(y) 
&=& \sum_i y_i \left( v_i \sum_j S_{ij} \frac{y_j d_j}{v_j d_i} \right) - \sum_i y_i^2 - \sum_{i \neq j} S_{ij} y_i y_j \\
&=& \sum_{i,j} S_{ij} y_i y_j \frac{v_i d_j}{v_j d_i} - \sum_i y_i^2 - \sum_{i \neq j} S_{ij} y_i y_j
\end{eqnarray*}

Using the symmetry of $S$ ($S_{ij} = S_{ji}$), we group the terms for each edge $(i,j) \in E$:
\[
\Delta_v(y) = \sum_{(i,j) \in E} S_{ij} y_i y_j \left( \frac{v_i d_j}{v_j d_i} + \frac{v_j d_i}{v_i d_j} - 2 \right).
\]

Factoring the expression in the parentheses:
\[
\Delta_v(y) = \sum_{(i,j) \in E} \frac{S_{ij} y_i y_j}{v_i v_j d_i d_j} \left( v_i d_j - v_j d_i \right)^2
\]

Since $S_{ij}, y_i, d_i, v_i > 0$, every term in the sum is non-negative, 
hence $\Delta_v(y) \ge 0$, which implies $y^T S y \le \sum v_i y_i$.

\textbf{Equality Case:}

The equality $\Delta_v(y) = 0$ holds if and only if each term in the sum vanishes. 
For all pairs $(i,j)$ with $S_{ij} > 0$, 
this requires $v_i d_j - v_j d_i = 0$, which implies $\frac{d_i}{v_i} = \lambda$ 
for some constant $\lambda > 0$ on each connected component of the support ($\lambda > 0$ because $d_i > 0$).
By the definition of the WRGN update, $z_i = \frac{y_i d_i}{v_i}$. 
Substituting $\frac{d_i}{v_i} = \lambda$ yields $z_i = \lambda y_i$ (or $z = \lambda y$ in vector form). 
Substituting this back into the definition of $d$ gives: $d = Sz = S(\lambda y) = \lambda Sy$. 
Consequently, the condition $d_i = \lambda v_i$ becomes $\lambda (Sy)_i = \lambda v_i$. 
Since $\lambda > 0$, we conclude $(Sy)_i = v_i$ for all $i \in \text{supp}(y)$. 
Hence $y$ is a fixed point of the WRGN dynamics, satisfying the weighted normalization condition.
\end{proof}

\begin{theorem}[Monotonic Growth of Weighted Mass]
\label{thm:weighted-mass-growth}
Let $S$ be a symmetric matrix with $S_{ii} = 1$ and $S_{ij} \ge 0$. 
Let $v \in \mathbb{R}^n_{++}$ be a vector of positive weights. Consider the WRGN sequence $\{y^k\}$ defined by the update rule:
\[
y_i^{k+1} = y_i^k \frac{v_i}{(Sy^k)_i}.
\]
If $y^k \in \mathbb{R}^n_{++}$ is the image of a prior WRGN step, 
then the weighted mass $M_v(y) := \sum_{i=1}^n v_i y_i$ is non-decreasing:
\[
M_v(y^{k+1}) \ge M_v(y^k).
\]
Furthermore, the increase is strict ($M_v(y^{k+1}) > M_v(y^k)$) unless $y^k$ is a fixed point of the dynamics.
\end{theorem}

\begin{proof}
For the current state $y^k$, we define a probability distribution $\lambda$ such that:
\[
\lambda_i = \frac{v_i y_i^k}{M_v(y^k)}, \quad \text{implying} \quad \sum_{i=1}^n \lambda_i = 1.
\]

The weighted mass at iteration $k+1$ is given by:
\[
M_v(y^{k+1}) = \sum_{i=1}^n v_i y_i^{k+1} = \sum_{i=1}^n v_i \left( y_i^k \frac{v_i}{(Sy^k)_i} \right) = \sum_{i=1}^n \frac{v_i^2 y_i^k}{(Sy^k)_i}.
\]

Factoring out the current mass $M_v(y^k)$, we rewrite the sum in terms of $\lambda_i$:
\[
M_v(y^{k+1}) = M_v(y^k) \sum_{i=1}^n \lambda_i \left( \frac{v_i}{(Sy^k)_i} \right)
\]

Since $f(t) = 1/t$ is a strictly convex function on $(0, \infty)$, we apply Jensen's inequality:
\[
\sum_{i=1}^n \lambda_i \frac{1}{(Sy^k)_i / v_i} \ge \frac{1}{\sum_{i=1}^n \lambda_i \frac{(Sy^k)_i}{v_i}}.
\]

Substituting the definition of $\lambda_i$ into the denominator:
\[
\sum_{i=1}^n \lambda_i \frac{(Sy^k)_i}{v_i} = \sum_{i=1}^n \left( \frac{v_i y_i^k}{M_v(y^k)} \right) \frac{(Sy^k)_i}{v_i} = \frac{1}{M_v(y^k)} \sum_{i=1}^n y_i^k (Sy^k)_i = \frac{(y^k)^T S y^k}{M_v(y^k)}
\]

Combining these results yields the bound:
\[
M_v(y^{k+1}) \ge M_v(y^k) \cdot \frac{M_v(y^k)}{(y^k)^T S y^k}.
\]

By Lemma \ref{lem:weighted-conjecture-81}, we have $M_v(y^k) \ge (y^k)^T S y^k$ for any $y^k$ that is an image of the WRGN operator. 
Therefore $M_v(y^{k+1}) \ge M_v(y^k)$. 

The inequality is strict unless all terms $(Sy^k)_i / v_i$ are equal for 
$i \in \text{supp}(y^k)$. If $(Sy^k)_i / v_i = \lambda$ for some constant $\lambda$, 
then by Lemma \ref{lem:weighted-conjecture-81}, $y^k$ is a fixed point and $\lambda = 1$. 
In all other cases, the strict convexity of $f(t) = 1/t$ ensures $M_v(y^{k+1}) > M_v(y^k)$. 
Furthermore, since $(y^k)^T S y^k < M_v(y^k)$ for all non-fixed points $y^k \in \text{im}(\text{WRGN})$, 
the growth factor $\frac{M_v(y^k)}{(y^k)^T S y^k}$ is strictly greater than unity, reinforcing the monotonic increase.
\end{proof}

\begin{corollary}
When applied to the MWIS problem with $v = \sqrt{w}$, 
WRGN strictly increases the MWIS objective $\sum w_i x_i$ at each iteration.
\end{corollary}
\begin{proof}
$M_v(y) = \sum_{i=1}^n v_i y_i = \sum_{i=1}^n v_i \times v_i x_i = \sum_{i=1}^n v^2_i x_i =\sum_{i=1}^n w_i x_i$.
\end{proof}

\section{Proof of Theorem~\ref{thm:binary-regularization} (Binarization Regularization)}
\label{app:proof-binary-regularization}

\textbf{Jacobian at Fixed Points.}

\begin{lemma}[Jacobian at Fixed Points of GN]
\label{lem:jacobian}
Given a graph $A$, $v>0$ and $\gamma>0$, we denote $B:=\CLO{A}^{\gamma, v}$ to lighten notations.
Let $x^*$ be a fixed point of the WRGN map $\GN{B}$ and $S := \supp(x^{*})$ the support of $x^*$.
Let $\{C_1, \dots C_p\}$ be the connected components of $S$.
The Jacobian $J(x^{*})$ of $\GN{B}$ at $x^*$ has the form:
\[
J_{ij}(x^{*}) = \frac{\big(\delta_{ij} - x_i^{*} B_{ij}\big)}{(B x^{*})_i}
\]

where $\delta_{ij}$ is the Kronecker delta.
It yields the upper-diagonal block structure:
\[
J(x^{*}) = \begin{pmatrix} 
J_{C_1 C_1} & 0 & \dots & J_{C_1\bar{S}} \\
0 & \ddots  & 0 & \vdots \\
0 & 0 & J_{C_p C_p} & J_{C_p\bar{S}} \\
0 & 0 & 0 & J_{\bar{S}\bar{S}} \end{pmatrix}
\]

where $J_{C_k C_k}$ corresponds to edges between nodes of the component $C_k$, 
$J_{\bar{S}\bar{S}}$ to edges between zero nodes, and 
$J_{C_k \bar{S}}$ to edges between the component $C_k$ and zero nodes.
The lower-left block is zero.

The diagonal blocks are:
\begin{align*}
J_{C_k C_k} &= I_{C_k} - \diag(x^*_{C_k})B_{C_k} \\
J_{\bar{S}\bar{S}} &= \diag\left(\frac{1}{B_{\bar{S}} x^*_{\bar{S}}}\right)
\end{align*}
\end{lemma}

\begin{proof}
$$
J_{ij}(x) = \frac{\partial}{\partial x_j}\Big[\frac{x_i}{(B x)_i}\Big] = \frac{\delta_{ij}(B x)_i - x_i B_{ij}}{(B x)_i^{2}} 
= \frac{1}{(B x)_i}\big(\delta_{ij} - \frac{x_i}{(B x)_i}B_{ij}\big)
$$

Therefore at a fixed point $x^*$, verifying $x^* = x^* / Bx^*$:
$$
J_{ij}(x^{*}) =\frac{\big(\delta_{ij} - x_i^{*} B_{ij}\big)}{(B x^{*})_i}
$$

For $i \in S$, we have $x_i^{*} > 0$ and thus $(Bx^*)_i = 1$, hence:
$$
J_{ij}(x^{*}) = \delta_{ij} - x_i^{*} B_{ij}
$$

Note that if $j$ also belongs to $S$, i.e. if $x^*_j >0$, then $J_{ij}(x^*)\ne 0$ if and only if $j$ is a neighbor of $i$ belonging to the same connected component $C_k$ of the support as $i$. Hence the diagonal block structure of $J_{SS}$.

For $i \notin S$, we have $x_i^{*} = 0$, and thus $J_{\bar{S}\bar{S}}$ is diagonal:
$$
J_{ij}(x^{*}) = \frac{ \delta_{ij}}{(B x^{*})_i}
$$

Note that $(B x^{*})_i > 0$ because $x^{*} \in \mathsf{N}_A$.
\end{proof}

We now consider a fixed point $x^*$ of the WRGN map $\GN{\CLO{A}^{\gamma, v}}$ with $\gamma > 1$. 
We prove the items of Theorem~\ref{thm:binary-regularization} one by one.

\textbf{Every non-binary fixed point is repulsive.}
Suppose that $x^*$ is not binary.
Its support must then contain a connected component $C$ with at least 2 nodes.
Indeed, if every node $i \in S$ satisfied $x^*_j = 0$ for all $j \in N(i)$, 
then $(Bx^*)_i = B_{ii}x^*_i = x^*_i$, which implies $x^*_i = x^*_i/x^*_i = 1$, 
meaning $x^*$ would be binary.

We analyze the Jacobian block $J_{CC}$. To avoid clutter, we restrict all notations
implicitly to the component $C$: $J = I - \diag(x^*) \CLO{A}^{\gamma, v}$.
Let $X = \diag(x^*)$. Since $J$ is similar to $M = X^{-1/2} J X^{1/2}$, they share 
the same eigenvalues. We have:
\[ 
M = I - X^{1/2} \CLO{A}^{\gamma, v} X^{1/2}.
\]

Let $Q = X^{1/2} \CLO{A}^{\gamma, v} X^{1/2}$. Since $\CLO{A}^{\gamma, v} = V^{-1}\CLO{A}^\gamma V$ 
and $V, X$ are positive diagonal matrices, $Q$ is congruent to $\CLO{A}^\gamma$. 
By Sylvester's Law of Inertia, $Q$ and $\CLO{A}^\gamma$ share the same number of negative eigenvalues. 
We examine the minimum eigenvalue of $\CLO{A}^\gamma$:
\[
\lambda_{\min}(\CLO{A}^\gamma) = \gamma \lambda_{\min}(A) + 1.
\]

For any \emph{connected} graph with at least one edge, $\lambda_{\min}(A) \leq -1$ \cite{cvetkovic1980spectra}. 
Thus, $\lambda_{\min}(\CLO{A}^\gamma) \leq 1 - \gamma < 0$ for $\gamma > 1$.
By congruence, $Q$ possesses at least one negative eigenvalue $\lambda_Q < 0$. 
The corresponding eigenvalue of the Jacobian block $M$ is $\lambda_M = 1 - \lambda_Q > 1$. 
Therefore, the spectral radius $\rho(J(x^*)) \ge \lambda_M > 1$, establishing that 
the fixed point is strictly repulsive.

\textbf{Condition of stability of binary fixed points.}
Now assume that $x^*$ is binary. 
It is necessarily a MIS hence all the connected components of its support $S=\supp(x^*)$ are isolated nodes.
For a node $i$ with $x_i=1$, the Jacobian bloc $J_{ii}$ corresponding 
to its support is a $1\times 1$ bloc with value:
$J_{ii} = I_{ii} - \diag(x^*)_{ii} (\CLO{A}^{\gamma, v})_{ii} = 1 - 1*1 = 0$.
The corresponding eigenvalue is thus $0$.

The non-support block $J_{\bar{S}\bar{S}}(x^*)$ is diagonal with entries for $i\notin S$: 
$J_{ii}(x^*) = \frac{1}{(\CLO{A}^{\gamma, v}x^*)_i} $
which are the eigenvalues of $J_{\bar{S}\bar{S}}(x^*)$.
Developing the denominator:
\[ 
(\CLO{A}^{\gamma, v}x^*)_i = \sum_{j=1}^n B^{\gamma, v}_{ij} x^*_j = \sum_{j \in S} \frac{v_j}{v_i} (\gamma A_{ij} + \delta_{ij}).
\]

Since $i \notin S$, the Kronecker delta $\delta_{ij} = 0$. 
Substituting $v = \sqrt{w}$:
\[
(\CLO{A}^{\gamma, v}x^*)_i = \gamma \sum_{j \in N(i) \cap S} \sqrt\frac{w_j}{w_i}.
\]

Therefore 
\[
\rho(J(x^*)) = \rho(J_{\bar{S}\bar{S}}(x^*)) = 
\max_{i \notin S} \left( \frac{1}{\gamma \sum_{j \in N(i) \cap S} \sqrt{\frac{w_j}{w_i}}} \right)
\]

Asymptotic stability requires $\rho(J(x^*)) < 1$ which gives the 
stability condition provided in the theorem.

\textbf{Stability of the MWISes.}
Let $S$ be is a MWIS.
Because the function $f(z) = \sqrt{z}$ is concave: $\sum \sqrt{w_j} \geq \sqrt{\sum w_j}$.
Therefore:
\begin{align*}
\stab_{\gamma, w}(S) &= \min_{i \notin S} \left( \gamma \sum_{j \in N(i) \cap S} \sqrt\frac{w_j}{w_i} \right)  \\
&= \min_{i \notin S} \left( \gamma \frac{\sum_{j \in N(i) \cap S} \sqrt{w_j}}{\sqrt{w_i}} \right)  \\
&\geq \min_{i \notin S} \left( \gamma \frac{\sqrt{\sum_{j \in N(i) \cap S} w_j}}{\sqrt{w_i}} \right)  \\
& = \gamma \min_{i \notin S} \sqrt{ \sum_{j \in N(i) \cap S} \frac{w_j}{w_i} }\\
\end{align*}

As $S$ is a MWIS, for any node $i \notin S$, 
the set $(S \cup \{i\}) \setminus (N(i) \cap S)$ is another independent set,
whose weight must be less than or equal to the weight of $S$:
\[
\forall i\notin S: \quad
w_i - \sum_{j \in N(i) \cap S} w_j \leq 0 \implies \forall i\notin S: \quad\sum_{j \in N(i) \cap S} \frac{w_j}{w_i} \geq 1.
\]

Hence:
\[
\stab_{\gamma, w}(S) \geq \gamma \min_{i \notin S} \sqrt{ \sum_{j \in N(i) \cap S} \frac{w_j}{w_i} }
\geq \gamma > 1.
\]

\textbf{Convergence to a MIS.}
As WRGN always converges and it must converge to a stable attractor, it necessarily 
converges to a MIS verifying the stability condition above.
As the MWISes are always stable (and also the global minimizers of the energy), 
there is at least one such stable attractor 
(conversely if it was not the case we would have a contradiction).

\section{Energy Landscapes and Phase Transitions}
\label{app:phase-transitions}
We explore $1$D energy profiles on the simplex as a function of $\gamma$ and $w$ in order
to illustrate the phase transition mechanism.

Consider the weighted complete graph $(K_2, w)$, with two nodes and a single edge between them.
We number the nodes $0$ and $1$ and parameterize their simplex state by $p_0=1-x$ and $p_1=x$,
so that $x=0$ corresponds to the binary solution $p=(1, 0)$ indicator of the MIS $M_0=\{0\}$,
$x=1$ to $p=(0,1)$ indicator of the MIS $M_1=\{1\}$,
and intermediate states $p=(1-x, x)$ correspond to fuzzy sets solutions.
As energy minima only depend on the ratio $w_0/w_1$, we fix the weight $w_1=1$.

The energy for this system is the quadratic function:
\begin{align*}
\mathcal{E}_{\gamma, w_0}(x) &= ((1-x)/\sqrt{w_0}, x) \begin{pmatrix} 1 & \gamma \\ \gamma & 1 \end{pmatrix} ((1-x)/\sqrt{w_0}, x) \\
&= x^2 + \frac{2\gamma}{\sqrt{w_0}} x(1-x) + \frac{1}{w_0}(1-x)^2
\end{align*}

The graph of $\mathcal{E}_{\gamma, w_0}$ is a parabola with its apex 
at $\mathcal{E}_{\gamma, w_0}'(x_a)=0$, which gives:
\[
x_a = \frac{1 - \gamma\sqrt{w_0}}{1+w_0-2\gamma\sqrt{w_0}}.
\]

The curvature of the parabola is $C=\mathcal{E}_{\gamma, w_0}''(x)=2-4\gamma/\sqrt{w_0}+2/w_0$.
It is zero, i.e. the parabola is a line, for $\gamma_F = \frac{1}{2}(\sqrt{w_0}+\frac{1}{\sqrt{w_0}})$.

The graphs of this energy are represented in Figure~\ref{fig:K2} for different values of $\gamma$ and $w_0$.
As states outside the simplex $\Delta_1 = [0,1]$ are prohibited, we represent them with infinite energies, which 
corresponds to energy barriers at $x=0$ and $x=1$.
Local minima (stable optima) of the energy are represented by black disks.
Local maxima (unstable optima) of the energy are represented by white disks.

\begin{figure}[h]
\centering
\caption{Energy profiles on $K_2$}
\label{fig:K2}
\setlength{\tabcolsep}{3pt}
\begin{tabular}{@{}lccccccc@{}}
$w_0$ & $\gamma=\frac{1}{4}$ & $\gamma=\frac{1}{2}$ & $\gamma=\frac{1}{\sqrt{2}}$ & $\gamma=1$ & $\gamma=\sqrt{2}$ & $\gamma=2$ & $\gamma=4$ \\
$1$ & 
\includegraphics[align=c, width=0.12\textwidth]{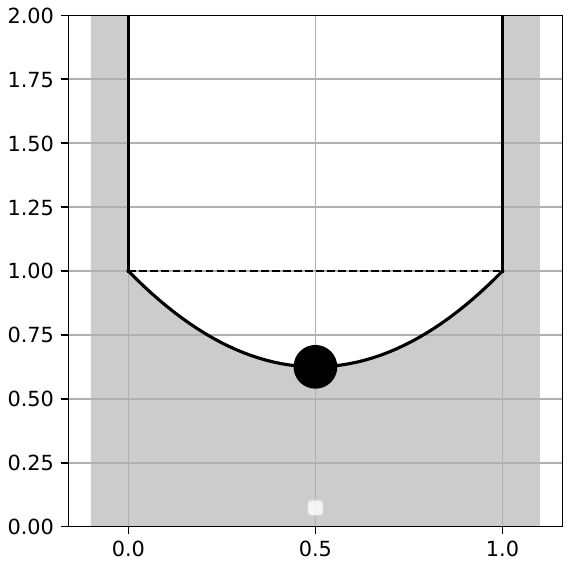} & 
\includegraphics[align=c, width=0.12\textwidth]{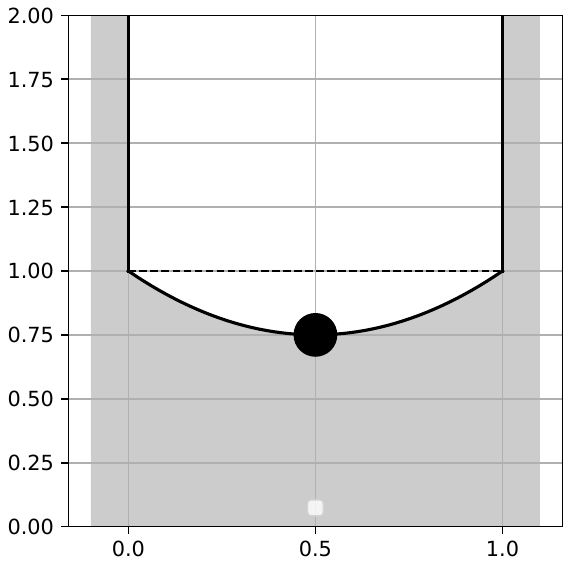} & 
\includegraphics[align=c, width=0.12\textwidth]{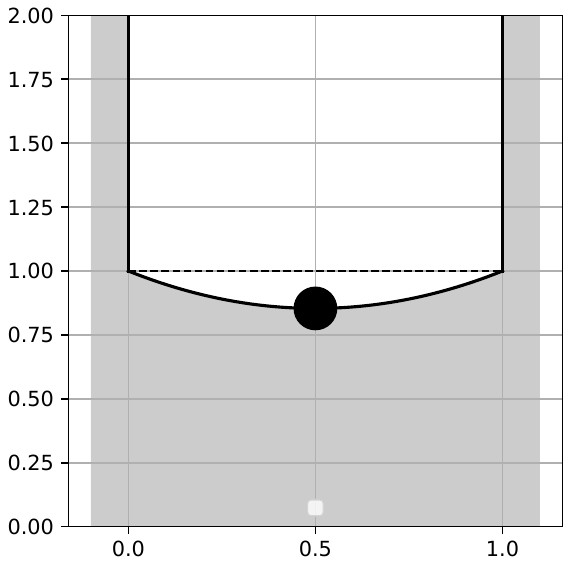} & 
\includegraphics[align=c, width=0.12\textwidth]{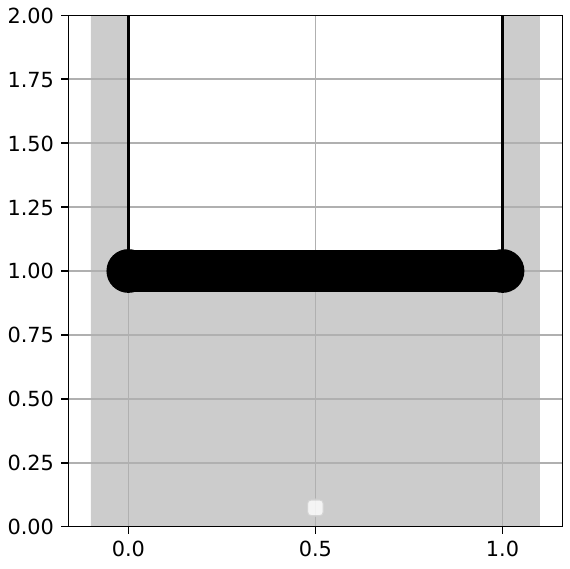} & 
\includegraphics[align=c, width=0.12\textwidth]{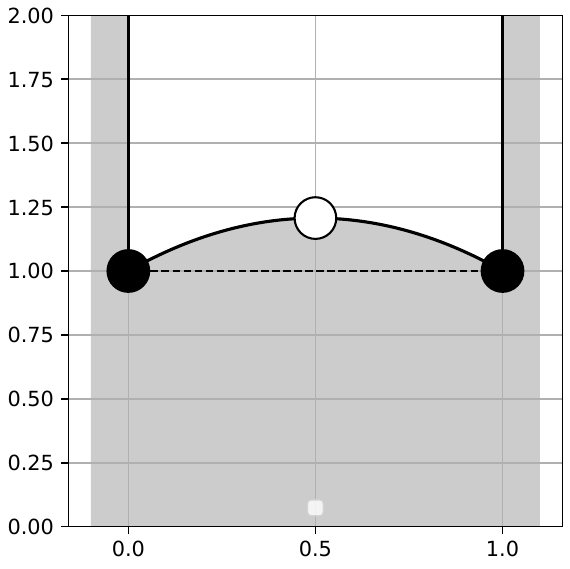} & 
\includegraphics[align=c, width=0.12\textwidth]{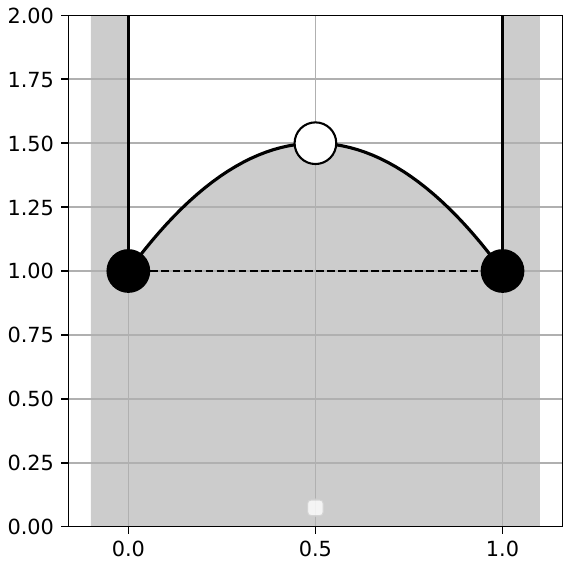} & 
\includegraphics[align=c, width=0.12\textwidth]{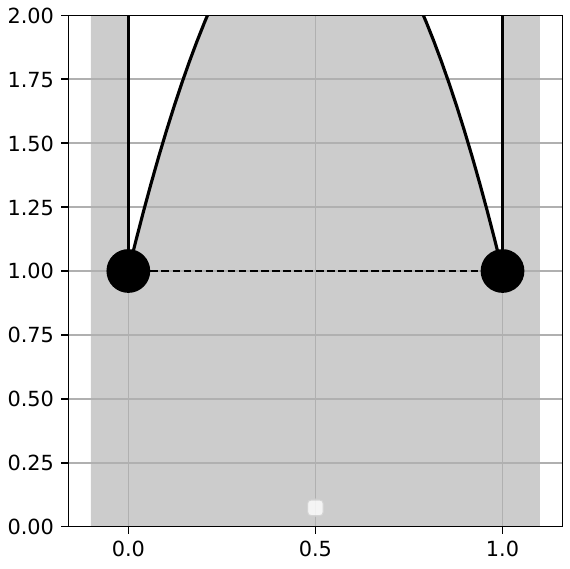} \\
$2$ &
\includegraphics[align=c, width=0.12\textwidth]{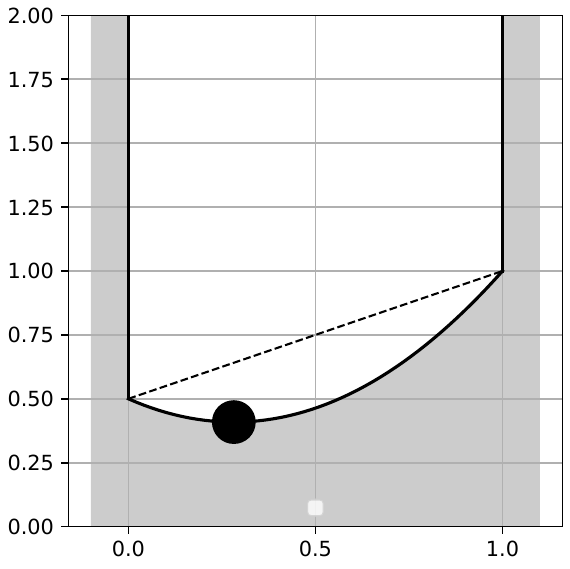} & 
\includegraphics[align=c, width=0.12\textwidth]{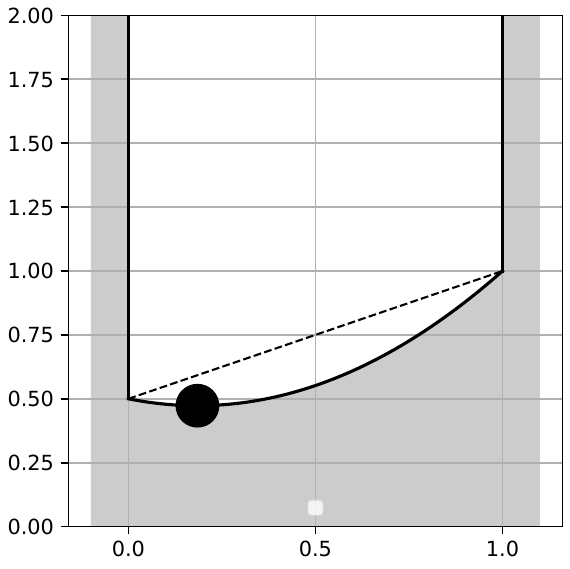} & 
\includegraphics[align=c, width=0.12\textwidth]{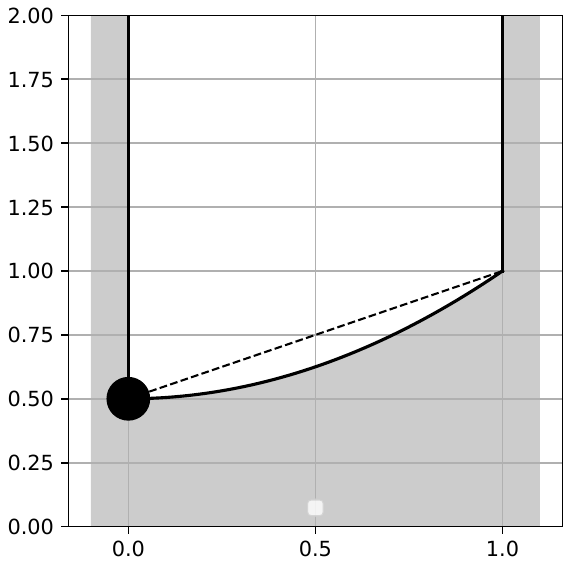} & 
\includegraphics[align=c, width=0.12\textwidth]{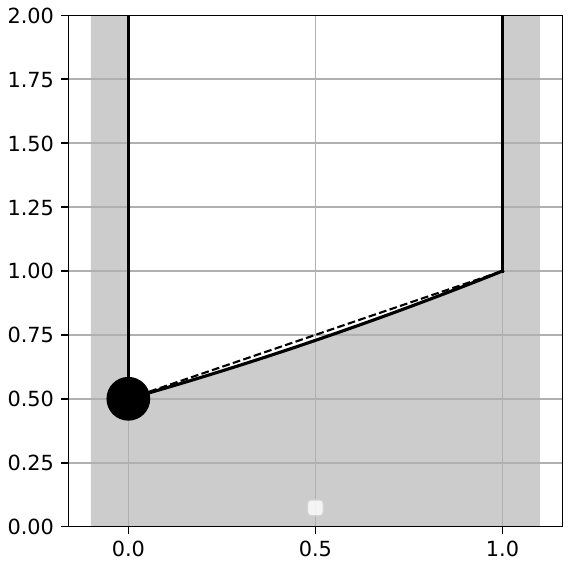} & 
\includegraphics[align=c, width=0.12\textwidth]{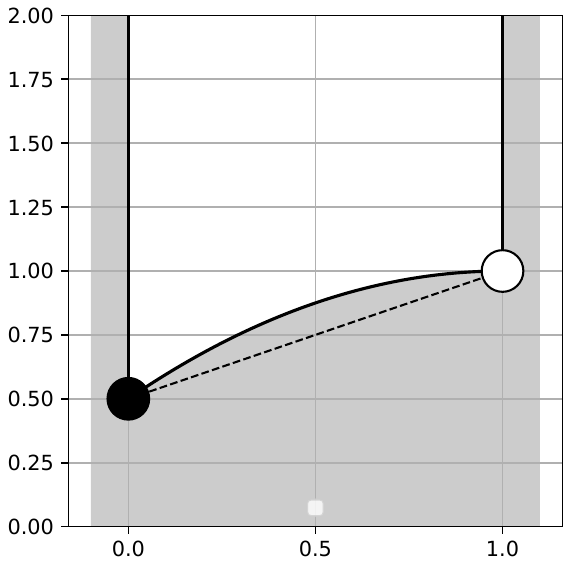} & 
\includegraphics[align=c, width=0.12\textwidth]{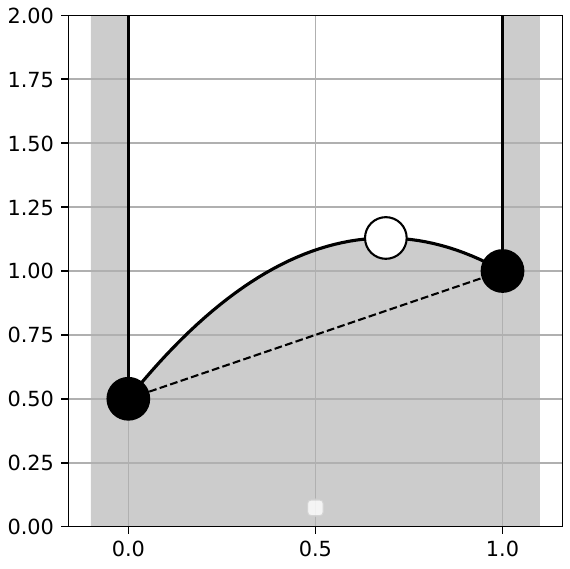} & 
\includegraphics[align=c, width=0.12\textwidth]{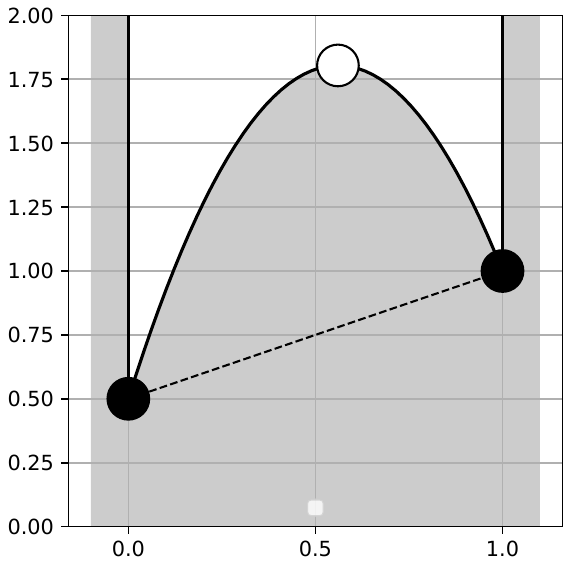} \\
$4$ & 
\includegraphics[align=c, width=0.12\textwidth]{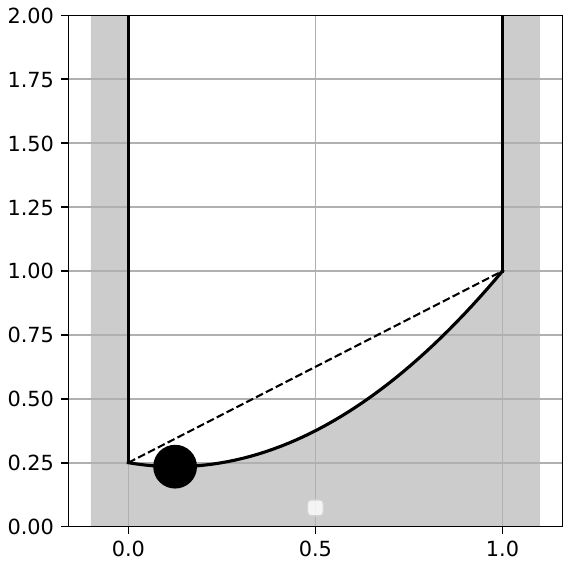} & 
\includegraphics[align=c, width=0.12\textwidth]{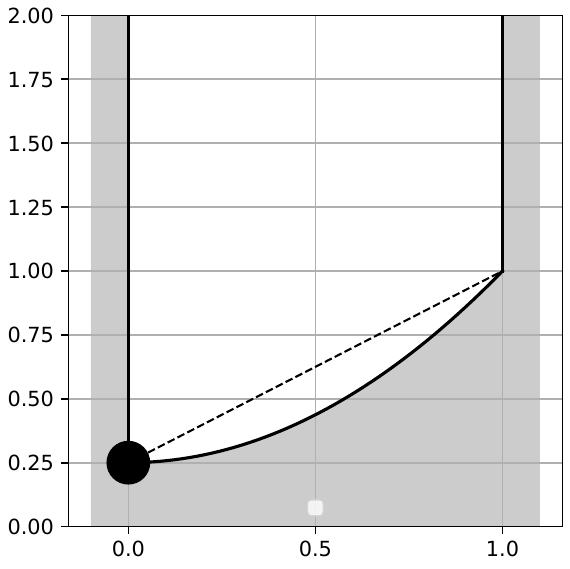} & 
\includegraphics[align=c, width=0.12\textwidth]{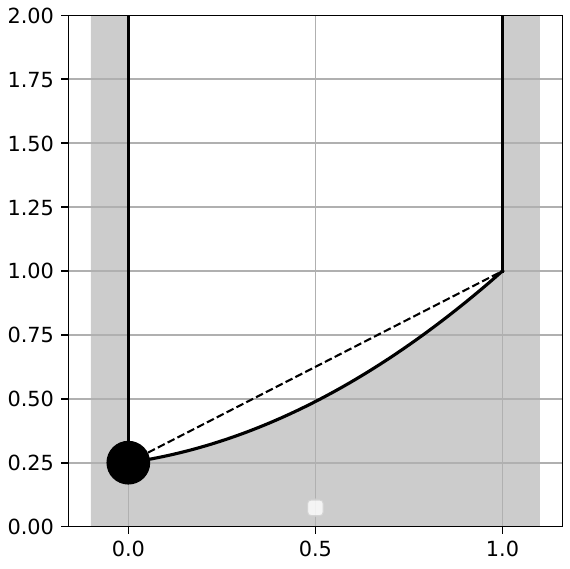} & 
\includegraphics[align=c, width=0.12\textwidth]{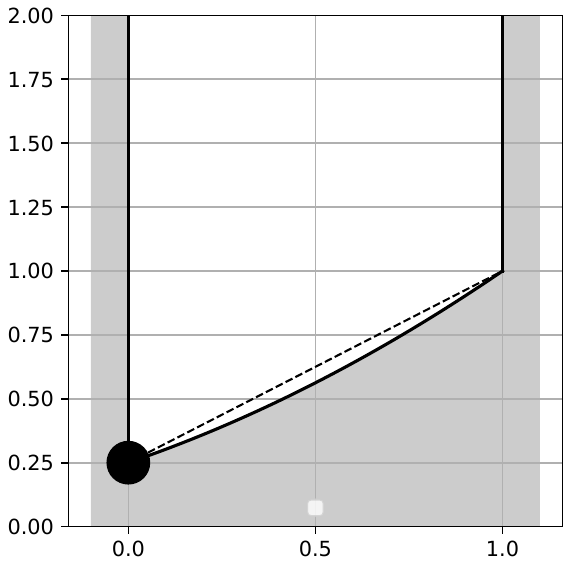} & 
\includegraphics[align=c, width=0.12\textwidth]{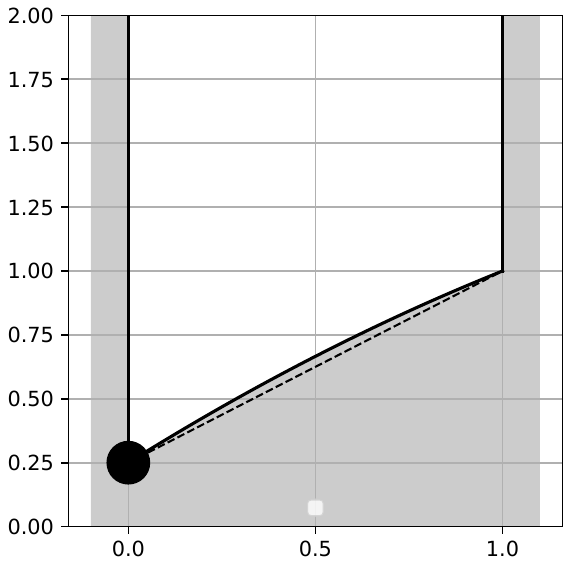} & 
\includegraphics[align=c, width=0.12\textwidth]{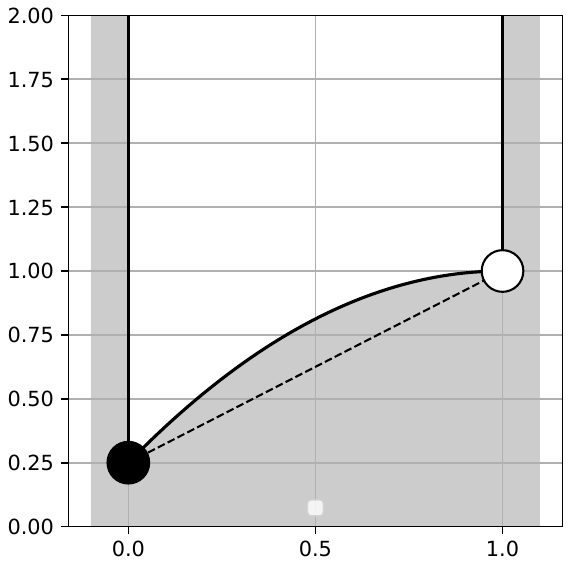} & 
\includegraphics[align=c, width=0.12\textwidth]{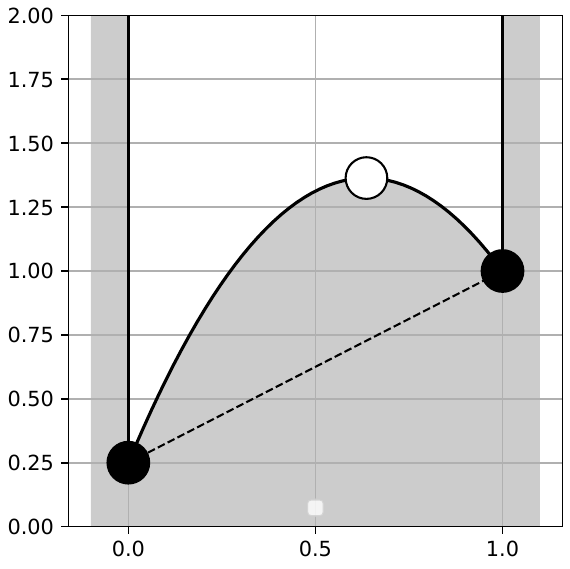} \\
\end{tabular}
\end{figure}

\textbf{1. Unweighted case ($w_0=1$).}

For $\gamma < 1$, the energy is strictly convex with a unique minimum at $x=1/2$.
Iterated GN converges to the fractional point $(1/2, 1/2)$ 
from any initialization\footnote{Note that $\gamma=0$ is a special case which corresponds to completely 
removing the edges from the graph, leaving the WRGN update: $x^{k+1}_i = x^{k}_i / (x^{k}_i + 0) = 1$.
As in WRGN, the domain constraints are enforced by the image of the map itself,
$\gamma = 0$ enforces the unique solution $x=\1$ which is the MIS of the 
graph with $n$ nodes and no edges.}.

For $\gamma = 1$, the energy is flat and any $x\in[0,1]$ a stable minimum.
Any positive initialization is projected to the simplex in $1$ iteration 
of graph normalization and stable after that.

For $\gamma > 1$, the energy is strictly concave, with two local minima at $x=0$ and $x=1$, 
which correspond the two binary MISes of the graph.
Iterated GN converges to $(1, 0)$ if initialized at $x^0 < 1/2$ and to $(0, 1)$ if 
initialized $x^0 > 1/2$. The apex $x_a=1/2$ is an unstable local maximum: 
initialization exactly on $1/2$ remains fixed but any slight deviation make the system fall 
to one of the two MISes of the graph.

\textbf{2. Weighted case ($w_0 > 1$).}

As show by the Tilted Simplex Motzkin-Straus Theorem~\ref{thm:motzkin-straus}, 
$w_0 \ne 1$ corresponds to tilting the simplex.
The energy value on the MIS at $x=0$ is $\mathcal{E}_{\gamma, w_0}(0) = 1/w_0$.

As $\gamma$ increases, the system goes through two phase transitions.

For $\gamma = 0$, the parabola is convex ($C = 2+2/w_0 > 0$) and its unique global minimum 
is at $x_a = \frac{w_0}{1+w_0} \in (0, 1)$. 
WRGN thus converges to the fractional solution $(1-x_a, x_a)$ from any initialization.

As $\gamma$ increases $x_a$ decreases until it crosses $0$ at $\gamma_B = 1/\sqrt{w_0}$.
It corresponds to the transition where WRGN now converges to the binary MIS $(1, 0)$ from any initialization.
For $w_0=2$, the transition happens at $\gamma_B = 1/\sqrt{2}$, and for $w_0=4$ at $\gamma_B = 1/2$, i.e. closer to $0$.

Further increasing $\gamma$, the curvature of the parabola crosses $0$ at $\gamma_F$, 
where it becomes a straight line before turning concave. 
However, because this line is tilted by the weight ratio, 
$x=0$ remains the unique minimum. 
A second local minimum (the suboptimal MIS) only emerges when the apex $x_a$ crosses $1$ at 
$\gamma_D = \sqrt{w_0}$. For $\gamma > \gamma_D$, the suboptimal MIS becomes $\gamma$-Stable, 
and the system exhibits the local-minimum correspondence predicted by Theorem~\ref{thm:motzkin-straus}.

\textbf{Generalization to other graphs.}

For a general weighted graph $(A,w)$, if we consider two adjacent nodes 
$i,j$, then the energy on the $1$D simplex slice given by $p_i = 1-x$, $p_j=x$ and $p_k = 0$ for $k\ne i,j$ 
has the same general profile than in the $K_2$ case and the same 
succession of transitions occur as $\gamma$ increases.
More generally, any $1$D slice of the simplex gives a $1$D quadratic energy which can 
have either $1$ or $2$ minima. To have $1$ fractional minimum, 
the quadratic function must be convex on the slice, with its minimum in $(0,1)$. 
On the contrary, if it has $2$ minima, they must be on the boundary of the domain 
and the energy must be concave along the slice.
In general, the energetical landscape is a combination of concave or convex profiles 
along $1$D directions.

\section{Proof of Theorem~\ref{thm:replicator} (WRGN as Nonlinear Replicator)}
\label{app:proof-RD}

By the definition of the simplex state $p^k$, 
we have $y_i^k = \frac{M_v^k p_i^k}{v_i}$,
or in vector form  $y^k = M_v^k p^k \oslash v$. 
Substituting this into the WRGN update rule:
\[
y_i^{k+1} = \left( \frac{M_v^k p_i^k}{v_i} \right) \frac{v_i}{(\CLO{A}^{\gamma} (M_v^k p^k \oslash v))_i} 
= \frac{M_v^k p_i^k}{M_v^k (\CLO{A}^{\gamma} (p^k \oslash v))_i} = \frac{p_i^k}{(\CLO{A}^{\gamma} (p^k \oslash v))_i}
\]

Multiplying both sides by $v_i$:
\[
v_i y_i^{k+1} = p_i^k \frac{v_i}{(\CLO{A}^{\gamma} (p^k \oslash v))_i} = p_i^k f_i(p^k)
\]

Summing over $i$ yields the evolution of the weighted mass:
\[
M_v^{k+1} = \sum_i p_i^k f_i(p^k) = \bar{f}(p^k)
\]
which shows that it is equal to the average fitness at the previous generation.
As WRGN systematically increases $M_v^k$ out of equilibrium, the average fitness follows the same evolution.

The normalized weighted state $p_i^{k+1} = v_i y_i^{k+1} / M_v^{k+1}$ then reduces to the discrete replicator equation:
\[
p_i^{k+1} = p_i^k \frac{f_i(p^k)}{\bar{f}(p^k)}.
\]

\textbf{WRGN is a non-potential game.}
For a vector field $f(x)$ to be a gradient (i.e., $f = \nabla \Phi$), 
the Poincaré Lemma requires that its Jacobian matrix must be symmetric:
$\frac{\partial f_i}{\partial x_j} = \frac{\partial f_j}{\partial x_i}$.
For the WRGN fitness function, letting $S=\CLO{A}^\gamma$ and $q=p\oslash v$, gives the condition: 
\begin{equation}
    -\frac{v_i}{v_j} \frac{S_{ij}}{(Sq)_i^2} = -\frac{v_j}{v_i} \frac{S_{ji}}{(Sq)_j^2}
\end{equation}

As $S$ is symmetric ($S_{ij} = S_{ji}$):
\begin{equation}\frac{v_i^2}{(Sq)_i^2} = \frac{v_j^2}{(Sq)_j^2} \implies \frac{(Sq)_i}{v_i} = \frac{(Sq)_j}{v_j}\end{equation}

Even in the unweighted case ($v=\1$), this requires $(Sx)_i = (Sx)_j$ for all adjacent nodes, 
a condition that cannot be satisfied globally (for all $x$) for non-trivial graphs. 
Non-uniform weights render the condition even harder.

This shows that the WRGN fitness is not a gradient hence the WRGN game is a non-potential game.

\section{Proof of Theorem~\ref{thm:motzkin-straus} (Weight-Tilted Simplex Motzkin-Straus)}
\label{app:proof-motzkin-straus}

\textbf{Correspondence between stable fixed points of WRGN and minima of energy.}

Since $\TENv$ is a smooth quadratic and WRGN is an exact MM algorithm,
we have the stationarity property: every limit point $y^* \in \NA$ 
is a fixed point of the map, which implies it is a critical point of $\TENv$ satisfying the KKT conditions 
(specifically, the complementarity condition $y_i [(\CLO{A}^\gamma y)_i - v_i] = 0$).
While the set of limit points is the set of critical points, 
only the stable fixed points of the map correspond to the local minima of the energy.
As proven by Theorem~\ref{thm:binary-regularization}, 
for $\gamma > 1$, the only asymtotically stable fixed points of WRGN are 
the $\gamma$-stable MISes.

Now, WRGN is undefined at the non-normalizable points of the non-negative orthant. 
However those non-normalizable points are not local minima of $\TENv$.
Indeed, a point $y \in \Rpn$ is non-normalizable if there exists at least one index $i$ 
such that $(\CLO{A}^\gamma y)_i = 0$. 
The partial derivative of the energy in the direction $i$ is then:
\[
\frac{\partial \TENv}{\partial y_i} = (\CLO{A}^\gamma y)_i - v_i = 0 - v_i = -v_i
\]
Since $v_i = \sqrt{w_i} > 0$, the gradient is strictly negative in the direction of $i$. 
As the constraint $y_i \geq 0$ is active at that boundary, the point cannot satisfy the first-order necessary conditions for a local minimum. 
Moving from $y$ to $y + \epsilon e_i$ (where $e_i$ is the $i$-th basis vector) 
strictly decreases the energy for any $\epsilon > 0$.

As a conclusion, for any $\gamma>1$, 
the local minima of the energy $\TENv(y)$ on $\Rpn$ are in $1:1$ correspondence with the 
$\gamma$-stable MISes of the graph.

\textbf{Minima on the Simplex.}

We establish the simplex form of the WRGN energy. 
For a weighted state $y\in \Rpn$, let $M_v = \sum_i v_i y_i$ 
denote the weighted mass and $p = (v \odot y) / M_v$ the simplex state, 
such that $y = M_v (p \oslash v)$. 
Substituting into the energy $\TENv(y) = \frac{1}{2}y^T \CLO{A}^\gamma y - v^T y$:
\begin{align*}
\TENv(M_v, p) &= \frac{M_v^2}{2} (p \oslash v)^T \CLO{A}^\gamma (p \oslash v) - M_v = \frac{M_v^2}{2} \mathcal{Q}_v(p) - M_v 
\end{align*}
where $\mathcal{Q}_v(p) := (p\oslash v)^T (I + \gamma A) (p \oslash v)$. 

For any fixed $p \in \Delta_{n-1}$, 
since $\mathcal{Q}_v(p) \ge 0$, 
the energy $\TENv(M_v, p)$ is a strictly convex quadratic function of $M_v$. 
The optimal mass along the ray defined by $p$ is found where $\frac{\partial\TENv}{\partial M_v} = 0$, 
yielding $M_v^*(p) = 1/\mathcal{Q}_v(p)$. Substituting this into the energy yields the reduced objective:
\[
\TENv(p) := \TENv(M_v^*(p), p) = -\frac{1}{2 \mathcal{Q}_v(p)}.
\]

Because $x \mapsto -1/x$ is a monotonically increasing function for $x>0$, 
minimizing $\TENv(p)$ is equivalent to minimizing $\mathcal{Q}_v(p)$.
Hence minimizing $\TENv$ over $\Rpn$ is equivelent to minimizing $\mathcal{Q}_v(p)$ over $\Delta_{n-1}$. 

\textbf{Minima on the Weight-tilted Simplex.} 
The change of variables $r = p \oslash v$ is a bijection between the standard simplex $\Delta_{n-1}$ 
and the weight-deformed manifold $\Delta^v_{n-1} := \{ r \in \Rpn \mid \sum_i v_i r_i = 1 \}$. 
Under this transformation, the quadratic form becomes:
\[
\mathcal{Q}_v(p) = (p \oslash v)^T \CLO{A}^\gamma (p \oslash v) = r^T \CLO{A}^\gamma r =: \mathcal{Q}(r).
\]
Consequently, minimizing the WRGN energy $\TENv(y)$ over $\Rpn$ is equivalent to finding:
\[
\min_{r \in \Delta^v_{n-1}} r^T (I + \gamma A) r.
\]

Hence the correspondence between $\gamma$-stable MISes of the graph and the local minima 
of $r^T (I + \gamma A) r$ over $\Delta^v_{n-1}$.

\section{Iterative Graph Normalization Pytorch Module}
\label{sec:appendix-code}

\begin{lstlisting}[language=Python, caption={Iterative Graph Normalization Pytorch Module}]
class IterativeGraphNormalization(nn.Module):
    def __init__(self, 
                 A: torch.sparse_coo_tensor, 
                 weights: np.ndarray, 
                 batch_size: int = 256, 
                 device: str = 'cpu'):
        super().__init__()
        self.v = torch.sqrt(weights).unsqueeze(0) 
        self.A = A
        self.batch_size = batch_size
        self.device = device
        # uniform initialization of the parameters on the simplex
        # by exponential sampling
        params = -torch.rand(self.batch_size, self.v.shape[0], device=self.device).log()
        params = torch.clamp(params / torch.max(params), 1e-3, 1.)
        self.params = nn.Parameter(params)
       
    def forward(self, gamma0: float = 0.9, gamma1: float = 1.5, iterations: int = 1000):        
        for iter in range(iterations):
            p = float(iter) / float(iterations-1)
            gamma = p*gamma1 + (1-p)*gamma0
            y = self.v * x 
            Ay = torch.sparse.mm(self.A, y.t()).t()
            Ay.mul_(gamma).add_(y)             
            # safe division
            mask = Ay > 1e-9
            x[mask] = y[mask] / Ay[mask]
            x[~mask] = 0.5           
            # explicitly clear Ay to help memory management
            del Ay
        return torch.clamp_(x, 0, 1)
\end{lstlisting}

\section{Complete experimental results}
\label{app:exp-results}

\begin{table*}[ht]
\centering
\caption{Performance on AVR dataset}
\begin{tabular}{@{}lccccccccc@{}}
\hline \\
& & & \multicolumn{2}{c}{\textbf{Random start}} 
& \multicolumn{2}{c}{\textbf{Warm start}} \\
\textbf{Instance} & \textbf{Nodes} & \textbf{Edges} & 
\textbf{E[Gap]} & \textbf{Best Gap} & 
\textbf{E[Gap]} & \textbf{Best Gap} & 
\textbf{Time BS} &
\textbf{Time GN} \\ \hline
AVR\_000 & 863.4K & 331.2M &- & - & - &\textbf{8.89\%} & $289$s & $22498.8$s\\
AVR\_001 & 881.0K & 342.2M &- & - & - &\textbf{9.37\%} & $309$s & $4867.6$s\\
AVR\_002 & 881.9K & 344.1M &- & - & - &\textbf{8.93\%} & $297$s & $3751.7$s\\
AVR\_003 & 578.2K & 219.7M &- & - & - &\textbf{8.36\%} & $110$s & $2244.0$s\\
AVR\_004 & 270.1K & 94.1M &- & - & - &\textbf{7.61\%} & $19$s & $1098.4$s\\
AVR\_005 & 602.5K & 194.8M &- & - & - &\textbf{7.56\%} & $50$s & $1987.6$s\\
AVR\_007 & 651.9K & 220.5M &- & - & - &\textbf{8.02\%} & $170$s & $1907.5$s\\
AVR\_008 & 381.4K & 115.1M &- & - & - &\textbf{7.16\%} & $75$s & $1621.5$s\\
AVR\_009 & 163.8K & 43.0M &- & - & -&\textbf{5.70\%} & $10$s & $801.4$s\\
AVR\_010 & 411.9K & 283.9M &- & - & - &\textbf{7.77\%} & $190$s & $1489.5$s\\
AVR\_014 & 127.9K & 78.5M &- & - & - &\textbf{6.67\%} & $10$s & $1748.6$s\\
AVR\_015 & 266.4K & 144.6M &- & - & - &\textbf{8.50\%} & $15$s & $1642.7$s\\
AVR\_016 & 194.4K & 111.1M &- & - & - &\textbf{9.29\%} & $14$s & $1726.5$s\\
AVR\_017 & 83.1K & 37.9M &- & - & - &\textbf{6.32\%} & $2$s & $1121.7$s\\
AVR\_018 & 83.8K & 38.8M &- & - & - &\textbf{6.03\%} & $35$s & $1164.7$s\\
AVR\_019 & 14.1K & 44.2K &6.65\% & \textbf{6.30\%} & 1.01\% &\textbf{0.99\%} & $<1$s & $0.8$s\\
AVR\_020 & 21.5K & 130.5K &8.10\% & \textbf{7.59\%} & 1.27\% &\textbf{1.23\%} & $1$s & $1.5$s\\
AVR\_021 & 27.6K & 236.0K &8.59\% & \textbf{8.15\%} & 1.37\% &\textbf{1.33\%} & $1$s & $2.9$s\\
AVR\_022 & 30.5K & 296.4K &8.79\% & \textbf{8.22\%} & 1.57\% &\textbf{1.55\%} & $1$s & $4.1$s\\
AVR\_023 & 15.6K & 126.8K &22.98\% & \textbf{21.09\%} & 0.03\% &\textbf{0.03\%} & $<1$s & $1.4$s\\
AVR\_024 & 979 & 3.1K &4.27\% & \textbf{3.10\%} & 0.00\% &\textbf{0.00\%} & $<1$s & $0.1$s\\
AVR\_025 & 10.9K & 505.4K &6.89\% & \textbf{5.63\%} & 0.98\% &\textbf{0.70\%} & $<1$s & $4.2$s\\
AVR\_026 & 10.9K & 604.0K &6.05\% & \textbf{4.71\%} & 0.48\% &\textbf{0.39\%} & $<1$s & $5.0$s\\
AVR\_027 & 1.0K & 2.4K &2.53\% & \textbf{1.74\%} & 0.06\% &\textbf{0.06\%} & $<1$s & $0.2$s\\
AVR\_028 & 12.3K & 515.9K &4.17\% & \textbf{3.32\%} & 0.60\% &\textbf{0.54\%} & $<1$s & $4.6$s\\
AVR\_029 & 12.3K & 553.9K &3.82\% & \textbf{3.16\%} & 0.54\% &\textbf{0.43\%} & $<1$s & $4.6$s\\
AVR\_030 & 4.0K & 13.6K &7.16\% & \textbf{6.03\%} & 0.93\% &\textbf{0.90\%} & $<1$s & $0.4$s\\
AVR\_031 & 20.1K & 606.3K &9.69\% & \textbf{8.92\%} & 2.95\% &\textbf{2.68\%} & $<1$s & $7.6$s\\
AVR\_032 & 33.6K & 1.9M &10.44\% & \textbf{9.32\%} & 3.11\% &\textbf{2.99\%} & $<1$s & $15.5$s\\
AVR\_033 & 47.5K & 4.0M &10.97\% & \textbf{10.33\%} & 3.02\% &\textbf{2.76\%} & $<1$s & $27.2$s\\
AVR\_034 & 3.1K & 22.7K &26.31\% & \textbf{22.77\%} & 0.00\% &\textbf{0.00\%} & $<1$s & $0.4$s\\
AVR\_035 & 10.8K & 485.3K &25.97\% & \textbf{23.76\%} & 5.69\% &\textbf{4.43\%} & $17$s & $4.0$s\\
AVR\_036 & 18.0K & 1.5M &27.41\% & \textbf{25.34\%} & 25.10\% &\textbf{23.09\%} & $18$s & $12.7$s\\
AVR\_037 & 22.3K & 2.3M &28.11\% & \textbf{26.53\%} & 26.02\% &\textbf{24.27\%} & $18$s & $17.8$s\\
\hline
\end{tabular}
\end{table*}

\begin{table*}[ht]
\centering
\caption{Performance on MSCD dataset}
\begin{tabular}{@{}lccccccccc@{}}
\hline \\
& & & \multicolumn{2}{c}{\textbf{Random start}} 
& \multicolumn{2}{c}{\textbf{Warm start}} \\
\textbf{Instance} & \textbf{Nodes} & \textbf{Edges} & 
\textbf{E[Gap]} & \textbf{Best Gap} & 
\textbf{E[Gap]} & \textbf{Best Gap} & 
\textbf{Time BS} &
\textbf{Time GN} \\ \hline
MSCD\_000 & 6.0K & 138.7K &1.37\% & \textbf{1.13\%} & 0.67\% &\textbf{0.53\%} & $1$s & $1.5$s\\
MSCD\_001 & 8.0K & 197.8K &1.91\% & \textbf{1.42\%} & 0.77\% &\textbf{0.62\%} & $1$s & $1.8$s\\
MSCD\_002 & 6.2K & 131.2K &1.18\% & \textbf{0.84\%} & 0.51\% &\textbf{0.44\%} & $1$s & $1.4$s\\
MSCD\_003 & 5.4K & 80.1K &3.85\% & \textbf{2.99\%} & 0.98\% &\textbf{0.59\%} & $1$s & $0.9$s\\
MSCD\_004 & 5.1K & 66.6K &4.42\% & \textbf{3.57\%} & 0.71\% &\textbf{0.67\%} & $1$s & $0.7$s\\
MSCD\_005 & 7.2K & 162.4K &3.45\% & \textbf{2.78\%} & 0.80\% &\textbf{0.64\%} & $1$s & $1.6$s\\
MSCD\_006 & 6.7K & 155.1K &5.26\% & \textbf{4.14\%} & 1.45\% &\textbf{0.94\%} & $1$s & $1.5$s\\
MSCD\_007 & 5.4K & 80.4K &4.31\% & \textbf{3.50\%} & 0.68\% &\textbf{0.53\%} & $1$s & $0.9$s\\
MSCD\_008 & 6.1K & 140.7K &1.49\% & \textbf{1.13\%} & 0.45\% &\textbf{0.45\%} & $1$s & $1.3$s\\
MSCD\_009 & 5.8K & 115.8K &1.32\% & \textbf{0.96\%} & 0.51\% &\textbf{0.41\%} & $1$s & $1.1$s\\
MSCD\_010 & 7.3K & 176.2K &2.04\% & \textbf{1.45\%} & 0.63\% &\textbf{0.54\%} & $1$s & $1.6$s\\
MSCD\_011 & 5.7K & 79.9K &4.72\% & \textbf{3.93\%} & 0.57\% &\textbf{0.50\%} & $1$s & $0.9$s\\
MSCD\_012 & 5.2K & 95.8K &0.98\% & \textbf{0.77\%} & 0.40\% &\textbf{0.40\%} & $1$s & $0.9$s\\
MSCD\_013 & 7.5K & 183.0K &1.97\% & \textbf{1.59\%} & 0.64\% &\textbf{0.51\%} & $1$s & $1.7$s\\
MSCD\_014 & 6.4K & 155.5K &0.91\% & \textbf{0.79\%} & 0.58\% &\textbf{0.49\%} & $1$s & $1.5$s\\
MSCD\_015 & 5.9K & 148.9K &7.56\% & \textbf{5.79\%} & 1.58\% &\textbf{1.34\%} & $1$s & $1.3$s\\
MSCD\_016 & 6.2K & 94.0K &4.88\% & \textbf{3.95\%} & 0.77\% &\textbf{0.55\%} & $1$s & $1.3$s\\
MSCD\_017 & 6.4K & 142.4K &1.81\% & \textbf{1.44\%} & 0.63\% &\textbf{0.55\%} & $1$s & $1.4$s\\
MSCD\_018 & 7.7K & 189.3K &1.63\% & \textbf{1.22\%} & 0.67\% &\textbf{0.58\%} & $1$s & $1.9$s\\
MSCD\_019 & 6.1K & 129.9K &1.19\% & \textbf{0.91\%} & 0.50\% &\textbf{0.43\%} & $1$s & $1.3$s\\
MSCD\_020 & 7.8K & 184.1K &3.26\% & \textbf{2.72\%} & 1.06\% &\textbf{0.89\%} & $1$s & $1.7$s\\
\hline
\end{tabular}
\end{table*}

\end{document}